\documentclass[twoside]{article}

\usepackage{involutive_format}
\usepackage{verbatim}

\usepackage[numbers]{natbib}
\usepackage{microtype}

\usepackage[hide=false,setmargin=true,marginparwidth=0.5in]{marginalia}

\usepackage{multicol}
\usepackage{}

\usepackage{graphicx} 
\usepackage{subfigure}

\usepackage{url}
\usepackage{bm}
\usepackage{framed}

\usepackage{amsmath}
\usepackage{amsfonts}
\DeclareMathAlphabet{\mathpzc}{OT1}{pzc}{m}{it}

\usepackage{stackrel}

\usepackage[shortlabels]{enumitem}

\usepackage{xcolor}
\usepackage{tikz}

\usepackage{rotating}

\usepackage{amsthm}
\usepackage{mathtools}

\usepackage{algorithm}
\usepackage{algpseudocode}

\usepackage[overlay]{textpos}

\theoremstyle{theorem}
\newtheorem{theorem}{Theorem}[section]

\theoremstyle{lemma}
\newtheorem{lemma}[theorem]{Lemma}

\theoremstyle{definition}
\newtheorem{definition}[theorem]{Definition}

\theoremstyle{conjecture}

\usetikzlibrary{calc,trees,positioning,arrows,chains,shapes.geometric,%
    decorations.pathreplacing,decorations.pathmorphing,shapes,%
    matrix,shapes.symbols}

\tikzset{
>=stealth',
  punktchain/.style={
    rectangle, 
    rounded corners, 
    draw=black, very thick,
    text width=10em, 
    minimum height=3em, 
    text centered, 
    on chain},
  line/.style={draw, thick, <-},
  element/.style={
    tape,
    top color=white,
    bottom color=blue!50!black!60!,
    minimum width=8em,
    draw=blue!40!black!90, very thick,
    text width=10em, 
    minimum height=3.5em, 
    text centered, 
    on chain},
  every join/.style={->, thick,shorten >=1pt},
  decoration={brace},
  tuborg/.style={decorate},
  tubnode/.style={midway, right=2pt},
}

\usepackage{hyperref}

\usepackage{dblfloatfix}
\usepackage{float}
\usepackage{inconsolata}

\newcommand{\IF}{\textrm{if~}}

\newcommand{\Nats}{\mathbb{N}}
\newcommand{\Reals}{\mathbb{R}}
\newcommand{\bI}{\mathbb{I}}
\newcommand{\Id}{\mathrm{Id}}
\newcommand{\bE}{\mathbb{E}}
\newcommand{\bV}{\mathbb{V}}
\newcommand{\cI}{\mathcal{I}}
\newcommand{\hcI}{\widehat{\mathcal{I}}}
\newcommand{\cJ}{\mathcal{J}}

\newcommand{\bigset}[1]{\bigl\{#1\bigr\}}

\newcommand{\bigabs}[1]{\bigl|#1\bigr|}
\newcommand{\Bigabs}[1]{\Bigl|#1\Bigr|}

\newcommand{\fcnerr}{{\epsilon_m}}
\newcommand{\neterr}{{\delta_m}}

\newcommand{\F}{R_{\fcnerr}}
\newcommand{\G}{S}
\newcommand{\aF}{\widehat{R}_{\fcnerr,\neterr}}
\newcommand{\aG}{\widehat{S}_{\neterr}}
\newcommand{\PI}{\pi}
\newcommand{\aINVNET}{\widehat{\mathcal{I}}_{\fcnerr,\neterr}}
\newcommand{\aINVNETr}{{\widehat{\mathcal{I}}_{\fcnerr,\neterr}}}
\newcommand{\h}{h_{\fcnerr}}
\newcommand{\ah}{\widehat{h}_{\fcnerr,\neterr}}
\newcommand{\g}{g_{\fcnerr}}
\newcommand{\INV}{\mathcal{I}_{\fcnerr}}
\newcommand{\INVr}{{\mathcal{I}_{\fcnerr}}}
\newcommand{\A}{A_{\fcnerr}}
\newcommand{\barA}{\bar{A}_{\fcnerr}}

\newcommand{\defn}[1]{\textbf{#1}}
\newcommand{\defas}{\coloneqq}

\newcommand{\conc}{%
  \mathord{
    \mathchoice
    {\raisebox{1ex}{\scalebox{.7}{$\frown$}}}
    {\raisebox{1ex}{\scalebox{.7}{$\frown$}}}
    {\raisebox{.7ex}{\scalebox{.5}{$\frown$}}}
    {\raisebox{.7ex}{\scalebox{.5}{$\frown$}}}
  }
}

\usepackage{hyperref}
\hypersetup{
  colorlinks=true,
  citecolor=blue,
  urlcolor=magenta,
}
\begin{document}

\twocolumn[
  \ourtitle{Deep Involutive Generative Models for Neural MCMC}
  \ourauthor{Span Spanbauer \And Cameron Freer \And Vikash Mansinghka}
  \ouraddress{MIT \And MIT \And MIT}

]

\begin{abstract}
We introduce deep involutive generative models, a new architecture for
deep generative modeling, and use them to define \emph{Involutive Neural MCMC},
a new approach to
fast neural MCMC. An involutive generative model represents a
probability kernel $G(\phi \mapsto \phi')$ as an involutive (i.e., self-inverting)
deterministic function $f(\phi, \pi)$ on an enlarged state space
containing auxiliary variables $\pi$. We show how to make these models volume preserving, and how to use deep
volume-preserving involutive generative models to make valid Metropolis--Hastings updates
based on an auxiliary variable scheme with an easy-to-calculate
acceptance ratio. We prove that deep involutive generative models and their volume-preserving special case are
universal approximators for probability kernels. This result
implies that with enough network capacity and training
time, they can be used to learn arbitrarily complex MCMC updates. We
define a loss function and optimization algorithm for training
parameters given simulated data. We also provide initial experiments showing that Involutive Neural MCMC can efficiently explore multi-modal distributions that are intractable for Hybrid Monte Carlo, and can converge faster than A-NICE-MC, a recently introduced neural MCMC technique.
\end{abstract}


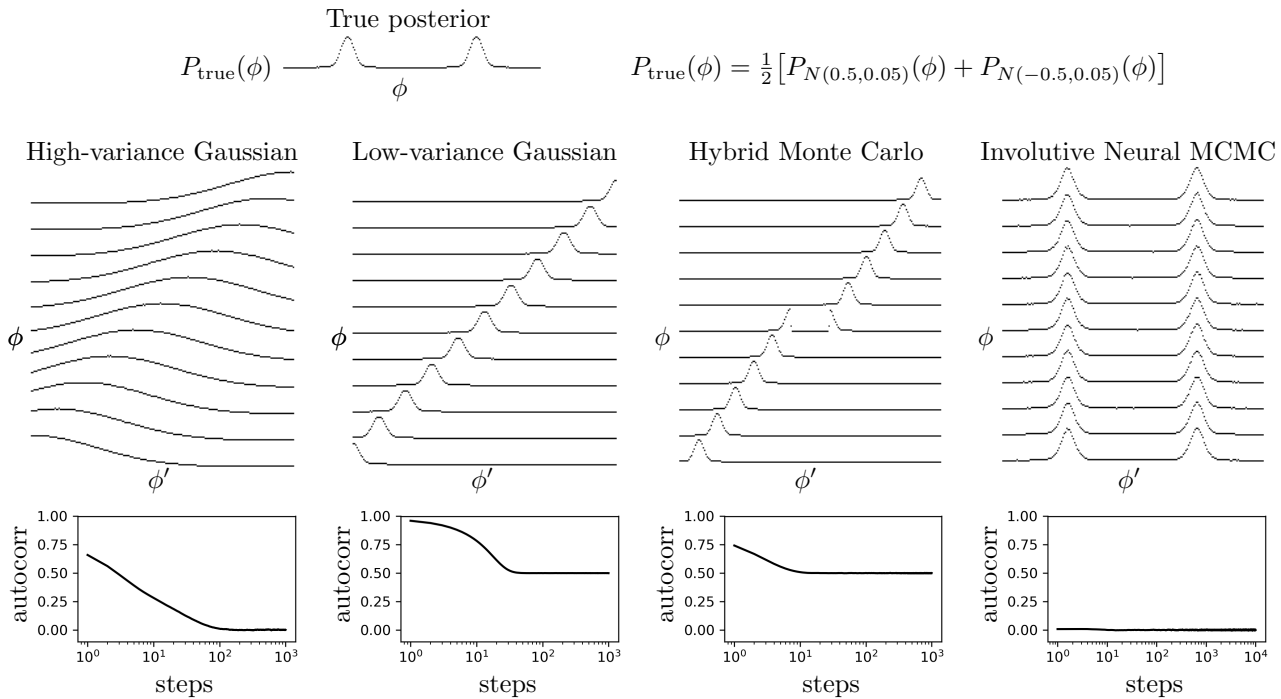
\begin{figure*}[!t]

\newcommand{\propx}{119}
\newcommand{\propy}{120}
\newcommand{\xaxx}{47}
\newcommand{\xaxy}{-6}
\newcommand{\yaxx}{-6}
\newcommand{\yaxy}{48}
\newcommand{\titley}{118}
\newcommand{\corrx}{119}
\newcommand{\corry}{70}
\newcommand{\corryaxx}{-5}
\newcommand{\corryaxy}{28}
\newcommand{\corrxaxx}{50}
\newcommand{\corrxaxy}{2}

\centering
\begin{picture}(80,40)
\put(-25,15){\includegraphics[width=35.5mm]{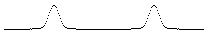}}
\put(18,7){$\phi$}
\put(-62,15){$P_{\mathrm{true}}(\phi)$}
\put(-7,33){True posterior}
\end{picture}
\begin{picture}(160,40)
    \put(25,15){$P_{\mathrm{true}}(\phi) = \frac{1}{2} \bigl[P_{N(0.5,0.05)}(\phi)+P_{N(-0.5,0.05)}(\phi)\bigr]$}
\end{picture}

\vspace{5mm}

\begin{picture}(\propx,\propy)
\put(0,0){\includegraphics[width=37mm]{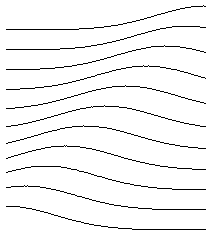}}
\put(\xaxx,\xaxy){$\phi'$}
\put(\yaxx,\yaxy){$\phi$}
\put(\yaxx,\yaxy){$\phi$}
\put(1,\titley){High-variance Gaussian}
\end{picture}
\begin{picture}(\propx,\propy)
\put(0,0){\includegraphics[width=37mm]{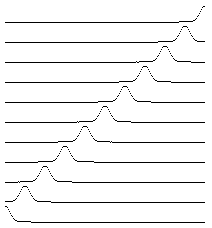}}
\put(\xaxx,\xaxy){$\phi'$}
\put(\yaxx,\yaxy){$\phi$}
\put(\yaxx,\yaxy){$\phi$}
\put(2,\titley){Low-variance Gaussian}
\end{picture}
\begin{picture}(\propx,\propy)
\put(0,0){\includegraphics[width=37mm]{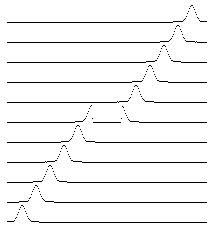}}
\put(\xaxx,\xaxy){$\phi'$}
\put(\yaxx,\yaxy){$\phi$}
\put(7,\titley){Hybrid Monte Carlo}
\end{picture}
\begin{picture}(\propx,\propy)
\put(0,0){\includegraphics[width=37mm]{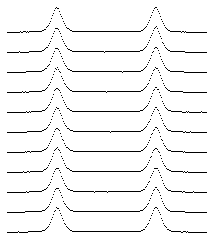}}
\put(\xaxx,\xaxy){$\phi'$}
\put(\yaxx,\yaxy){$\phi$}
\put(-5,\titley){Involutive Neural MCMC}
\end{picture}

\vspace{5mm}

\begin{picture}(\corrx,\corry)
\put(-5,0){\includegraphics[width=40mm]{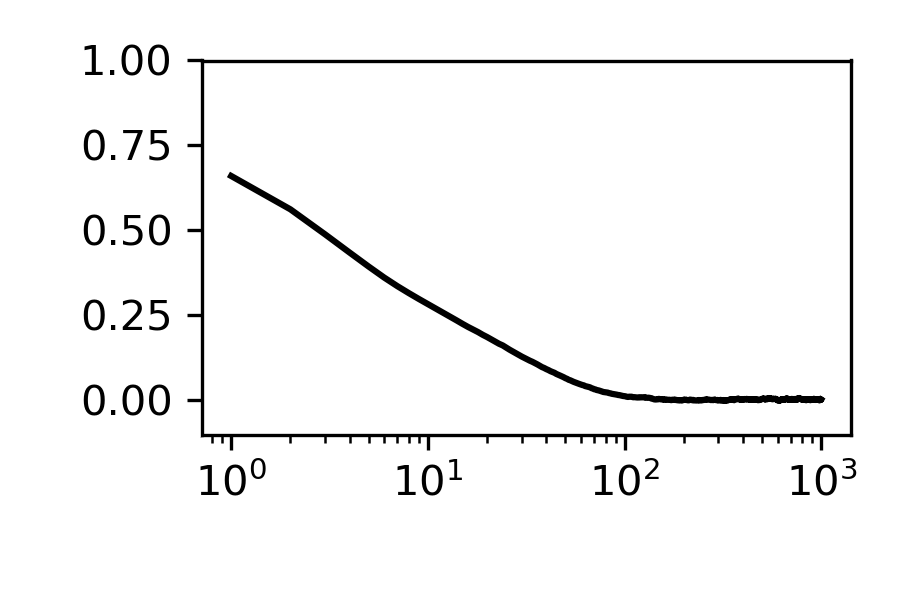}}
\put(\corrxaxx,\corrxaxy){steps}
\put(\corryaxx,\corryaxy){$\rotatebox{90}{autocorr}$}
\end{picture}
\begin{picture}(\corrx,\corry)
\put(-5,0){\includegraphics[width=40mm]{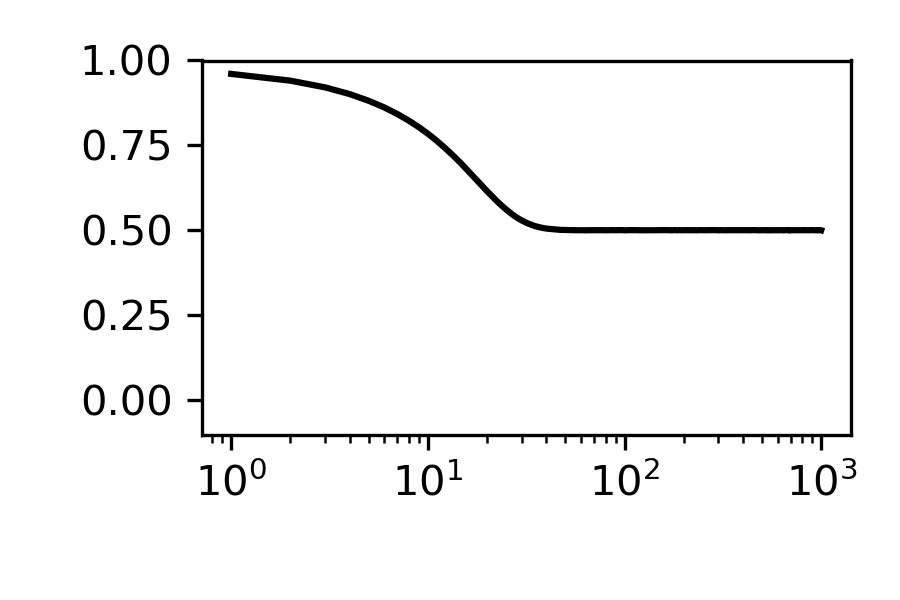}}
\put(\corrxaxx,\corrxaxy){steps}
\put(\corryaxx,\corryaxy){$\rotatebox{90}{autocorr}$}
\end{picture}
\begin{picture}(\corrx,\corry)
\put(-5,0){\includegraphics[width=40mm]{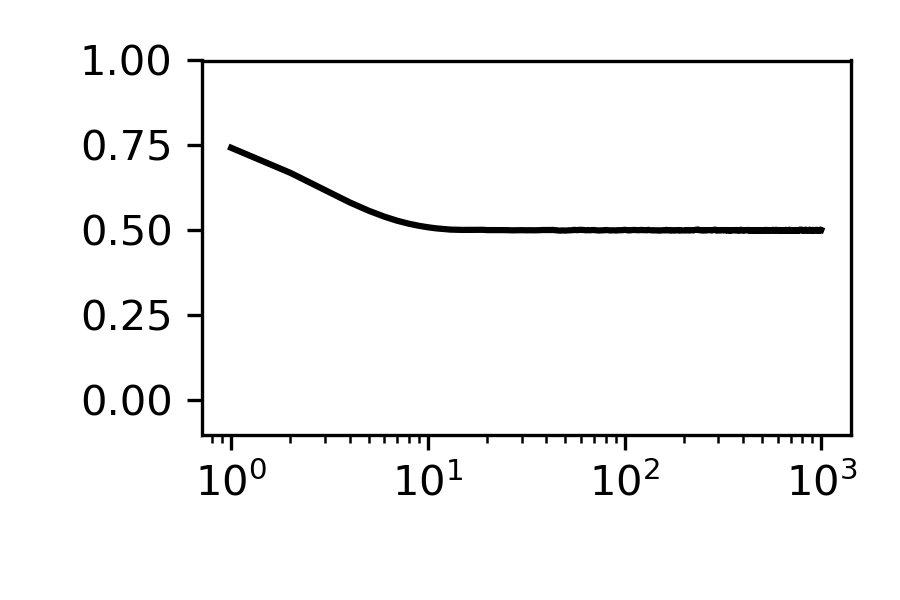}}
\put(\corrxaxx,\corrxaxy){steps}
\put(\corryaxx,\corryaxy){$\rotatebox{90}{autocorr}$}
\end{picture}
\begin{picture}(\corrx,\corry)
\put(-5,0){\includegraphics[width=40mm]{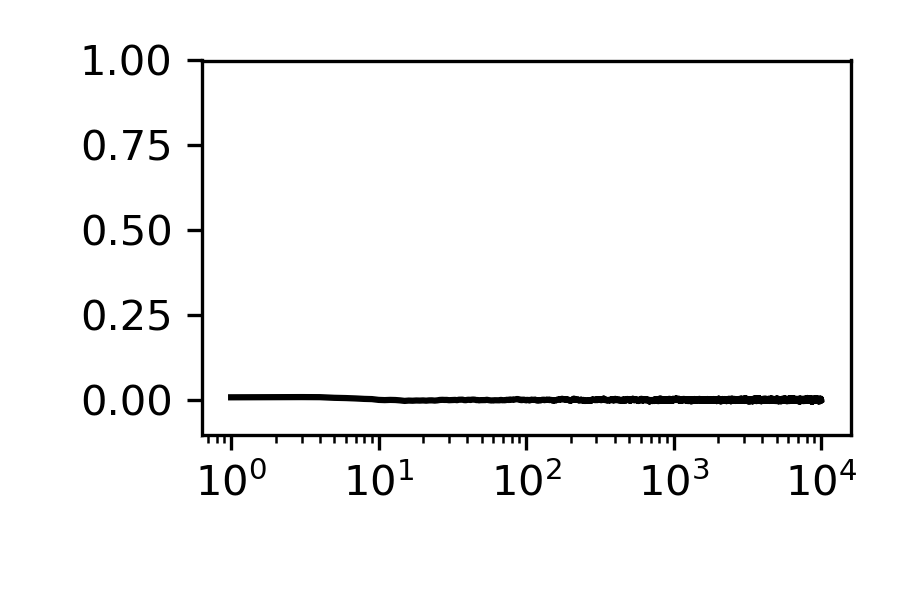}}
\put(\corrxaxx,\corrxaxy){steps}
\put(\corryaxx,\corryaxy){$\rotatebox{90}{autocorr}$}
\end{picture}

\vspace{3mm}

\caption{Consider the problem of using MCMC to sample from a mixture of two Gaussians. Here, each trace shows
the distribution (rescaled for clarity) of proposed state transitions $\phi'$ from a given initial state $\phi$.
High-variance Gaussian proposals are a poor approximation of the posterior, and hence converge slowly. Low-variance Gaussian proposals fail to mix between the two modes because proposals between the modes will be rejected with high probability. Hybrid Monte Carlo converges quickly within a mode, but also fails to mix between modes. Proposals from Involutive Neural MCMC nearly match the posterior from every state, mixing and nearly converging in a single step.}
\end{figure*}

\section{Introduction}

Markov Chain Monte Carlo (MCMC) methods are a class of very general
techniques for statistical inference \citep{brooks2011handbook}. MCMC
has seen widespread use in many domains, including cosmology \citep{dunkley2005fast}, 
localization \citep{pak2015improving}, phylogenetics \citep{ronquist2012mrbayes}, 
and computer vision \citep{kulkarni2015picture}. For the Metropolis--Hastings class of
MCMC algorithms, one must specify a proposal distribution $q(\phi \mapsto
\phi')$. Convergence will be slow if the proposal distribution is poorly
tuned to the posterior distribution $p(\phi'|D)$. Conversely, if one could
use a perfectly-tuned proposal distribution --- ideally, the exact
posterior $q(\phi \mapsto \phi') = p(\phi'|D)$ --- then MCMC would converge
in a single step.

Neural MCMC refers to an emerging class of deep learning approaches
\citep{song2017nice,wang2018meta,levy2017generalizing} that attempt to
learn good MCMC proposals from data. Neural MCMC approaches can be
guaranteed to converge, as the number of MCMC iterations increases,
to the correct distribution ---
unlike neural variational
inference \citep{DBLP:conf/icml/MnihG14, DBLP:journals/corr/KingmaW13, DBLP:conf/icml/RezendeMW14}, which can suffer from biased
approximations. Recently, \citet{song2017nice} suggested that
involutive neural proposals are desirable but difficult to achieve:
\vspace*{-5pt}
\begin{quote}
``If our proposal is deterministic, then $f_\theta(f_\theta(x, v)) = (x, v)$ should hold for all $(x, v)$, a condition which is difficult to achieve." \citep{song2017nice}
\end{quote}
\vspace*{-5pt}

\noindent {\bf Contributions.} This paper presents a solution to the
problem of learning involutive proposals posed by
\citep{song2017nice}. Specifically, it presents the following
contributions:
\vspace*{-5pt}
\begin{enumerate}
\item This paper introduces involutive neural networks, a new class of
  neural networks that is guaranteed to be involutive by construction; we also show how to constrain the Jacobian of these networks to have magnitude 1, that is, to preserve volume.

\item This paper uses involutive networks to define involutive
  generative models, a new class of auxiliary variable models, and
  shows that the volume-preserving ones can be used as Metropolis--Hastings proposals.

\item This paper proves that volume-preserving involutive generative models are
  universal approximators for transition kernels,
  justifying their use for black-box learning of good MCMC proposals.

\item This paper describes a new, lower-variance estimator for the Markov-GAN training objective \citep{song2017nice} that we use to train involutive generative models.

\item This paper shows that \emph{Involutive Neural MCMC} can improve
  on the convergence rate of A-NICE-MC, a state-of-the-art
  neural MCMC technique.

\item This paper illustrates \emph{Involutive Neural MCMC} on a simple problem.

\end{enumerate}
\vspace*{-5pt}
We motivate our approach by showing that several common Metropolis--Hastings proposals are special cases of involutive proposals (Section~\ref{sec:background}). We then show that by using a class of exactly involutive neural network architectures (Section~\ref{sec:involutive-nets}) satisfying an appropriate universality condition (Section~\ref{sec:universality}) and using adversarial training (Section~\ref{sec:training}), we can find involutive proposals that empirically converge extremely rapidly (Section~\ref{sec:experiments}).

\section{Background}
\label{sec:background}
\vspace{-5pt}
Recall that the speed of convergence of a Metropolis--Hastings algorithm is highly dependent on how well the proposal distributions match the posterior.

In order to use a given proposal distribution, one typically constructs a transition which satisfies the \emph{detailed balance} condition, which (in an ergodic setting) ensures convergence to the posterior. Satisfying this condition for a general proposal is hard, which has led researchers to use smaller classes of proposals for which this problem is tractable. Our method, \emph{Involutive Neural MCMC}, satisfies detailed balance for a universal class of proposal distributions, drawn from a generative model specified by a volume-preserving involutive function. Our method builds on previous work on invertible neural networks \citep{ardizzone2018analyzing}, for example the architecture we use in our constructive proof of universality makes use of additive coupling layers \citep{dinh2014nice} which have been cascaded \citep{dinh2016density,jacobsen2018revnet}. We now describe several existing classes of proposals, and observe that each can be viewed as an involutive proposal.

The canonical example of a class of proposal distributions is the collection of shifts by a multivariate Gaussian. These immediately satisfy detailed balance due to their symmetry, that is, the probability
$P(\phi \mapsto \phi')$
of a forward transition is equal to the probability
$P(\phi' \mapsto \phi)$
of a backward transition.
However, multivariate Gaussians are usually poor approximations of the posterior, leading to slow convergence. We observe that these proposals can be viewed as involutive proposals: choose the auxiliary variable $\pi$ to be a sample from the multivariate Gaussian, and define the state transition to be $(\phi,\pi) \mapsto (\phi+\pi,-\pi)$.

Another example class of proposal distributions is those generated by Hamiltonian dynamics in the Hybrid Monte-Carlo algorithm \citep{duane1987hybrid}. These proposals can be shown to satisfy detailed balance because they are involutive: proposals for Hybrid Monte Carlo are obtained by simulating a particle for a certain amount of time and then negating its momentum; if one performs this operation twice, the particle will end in its initial state.

Recently, researchers have begun using neural networks to parameterize classes of proposal distributions, leading to \emph{neural MCMC} algorithms. The A-NICE-MC method involves choosing a symmetric class of proposals parameterized by an invertible neural network: its Metropolis--Hastings proposal assigns 1/2 probability to the output of the network, and 1/2 to the output of the inverse of the network. This proposal is symmetric, and hence satisfies detailed balance. However, one can also view it as involutive. Specifically, let $f$ be an invertible neural network, and $\pi \sim N(0,1)$ the auxiliary variable. Define the state transition to be
$(\phi,\pi) \mapsto \begin{cases}
(f(\phi),-\pi) & \mathrm{if~} \pi>0, \\
(f^{-1}(\phi),-\pi) & \mathrm{otherwise}.
\end{cases}$

We have seen that all of these examples are special cases of involutive proposals. We now introduce
a class of exactly involutive neural network architectures (Section~\ref{sec:involutive-nets}) and
show that they satisfy a universality condition, and so may be used to
approximate any involutive proposal arbitrarily well (Section~\ref{sec:universality}).

\begin{figure}[t!]
\centering
\begin{picture}(160,65)
\put(0,0){\includegraphics[width=2.75cm]{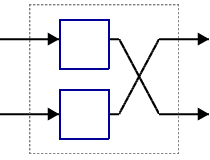}}
\put(78,0){\includegraphics[width=2.75cm]{figures/inv_fcn_block.png}}
\put(-22,43){$x_{1..n}$}
\put(160,43){$x_{1..n}$}
\put(-37,15){$x_{n+1..2n}$}
\put(160,15){$x_{n+1..2n}$}
\put(27,42){$g_\theta$}
\put(24.5,14.5){$g_\theta^{-1}$}
\put(104.5,42){$g_\theta$}
\put(102.5,14.5){$g_\theta^{-1}$}
\end{picture}
\caption{System diagram showing that two composed involutive function blocks forms the identity operation.}
\label{fig:system-diagram}
\end{figure}

\section{Involutive Neural Networks}
\label{sec:involutive-nets}

In this section, we describe how to build deep neural networks which are exactly involutive by construction. To do this, we first describe three kinds of smaller involutive building blocks, and then we describe how to compose these blocks to form a deep involutive network:

\begin{itemize}
\item \textbf{Involutive function blocks}, which are fairly general nonlinear maps, but do not fully mix information, in that every element of the output is independent of half of the elements of the input.
\item \textbf{Involutive permutation blocks}, which are linear maps and cannot be optimized, but can mix information.
\item \textbf{Involutive matrix blocks}, which are linear maps, but can be optimized and can mix information.
\end{itemize}

By composing these blocks in a particular way, we can create deep networks which have high capacity and are exactly involutive. We further show that each involutive block is either volume-preserving or can be made so, leading to high-capacity volume-preserving involutive networks appropriate for use as MCMC transition kernels. In Section~\ref{sec:universality}, we prove that these deep involutive networks are universal in a particular sense.

Let $a \conc b$ denote the concatenation of vectors $a$ and $b$, and write $a_{j..k}$ to denote the restriction of the vector $a$ to its terms indexed by $j, j+1, \ldots, k$. Let $\Id_n$ denote the $n\times n$ identity matrix.

\subsection{Involutive function blocks}

Involutive function blocks enable the application of fairly general nonlinear functions to the input data. In the typical case, they are parameterized by an invertible neural network \citep{ardizzone2018analyzing}, which is itself parameterized by two neural networks of arbitrary architecture.

\begin{definition}[Involutive function block]
Let $n \in \Nats$ and let $g \colon \Reals^n \to \Reals^n$ be a bijection.

Define the \defn{involutive function block}
$I_F^{2n,g} \colon \Reals^{2n} \to \Reals^{2n}$ by
$I_F^{2n,g}(x) \defas  g^{-1}(x_{n+1..2n}) \conc g(x_{1..n})$.
\end{definition}

\begin{figure*}[t!]
\centering
\includegraphics[width=\textwidth]{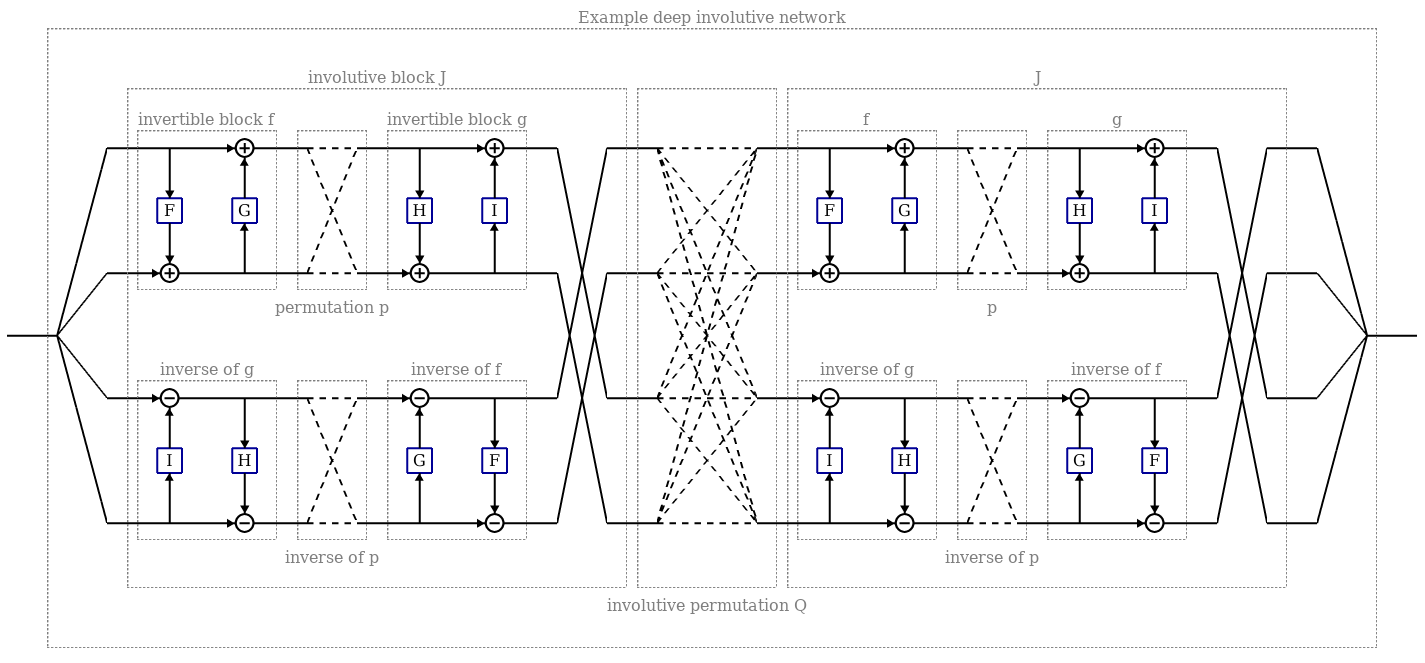}
\caption{System diagram of a typical deep involutive network. The functions $F$, $G$, $H$, and $I$ are arbitrary functions, usually induced by neural networks.}
\label{fig:system-typical}
\end{figure*}

Observe that
$(I_F^{2n,g} \circ I_F^{2n,g})(x) = (g^{-1} \circ g)(x_{1..n}) \conc$ \linebreak $(g \circ g^{-1})(x_{n+1..2n}) = x$, where $\circ$ denotes function composition,
and so the function $I_F^{2n,g}$ is indeed an involution. See Fig.~\ref{fig:system-diagram} for a system diagram depicting this fact.
Further note that $n$ and $g$ are uniquely determined by the function $I_F^{2n,g}$, and so when this function is induced by some neural net (i.e., when $g$ is itself induced by a neural net), we may think of $I_F^{2n,g}$ as a neural net with the same parameters as the neural net inducing $g$.
In this case we will sometimes elide the distinction between the function and this neural net, or refer to the function induced by the neural net by the same symbol.

If the parameter $g$ is volume-preserving, that is, the determinant of its Jacobian has magnitude 1, then we observe that the resulting involutive function block is also volume-preserving by considering the properties of the determinant of block matrices. In our experiments we obtain volume-preserving involutive function blocks by parameterizing $g$ using NICE additive coupling layers \citep{dinh2014nice} which have been cascaded \citep{dinh2016density,jacobsen2018revnet}, since these NICE layers preserve volume.

\subsection{Involutive permutation blocks}

One can see from Fig.~\ref{fig:system-diagram}, as previously noted, that each output of an involutive function block is independent of half of the inputs. In order to create more general involutive networks without this property, we mix information by applying an involutive permutation.

\begin{definition}[Involutive permutation block]
Let $n \in \Nats$ and let $\sigma$ be an involution on the set $\{1,\ldots, n\}$. Let $\bm{\sigma}$ denote the matrix defined by $\bm{\sigma}_{ij}e_j = e_{\sigma(i)}$ for $i, j\in \{1, \ldots, n\}$, where the $e_i$ are basis vectors.

Define the \defn{involutive permutation block}
$I_P^{n,\sigma} \colon \Reals^{n} \to \Reals^{n}$
by
$I_P^{n,\sigma} (x) \defas \bm{\sigma} x$.
\end{definition}

Note that $n$ and $\sigma$ are uniquely determined by the function $I^{n, \sigma}_P$. We may also think of this function as a linear layer with no parameters in a neural net.

Observe that involutive permutation blocks, as permutations, are volume preserving.

One may use any involutive permutation $\sigma$: we use a specific choice of $\sigma$ in the proof of universality, and we use uniformly random involutive permutations in the experiments. An algorithm for sampling uniformly from the space of $n$-dimensional involutive permutations is described in \citep{arndt2010generating}.

\subsection{Involutive matrix blocks}

As an alternative to involutive permutations, we may use a different class of involutive matrices to mix information. Compared to involutive permutations, involutive matrix blocks have the advantage that they can be optimized. This is because they are parameterized by two arbitrary nonzero vectors of real numbers.

\begin{definition}[Involutive matrix block]
    Let $n \in \Nats$ and let $v, w \in \Reals^n \setminus \{0^n\}$.

Define the \defn{involutive matrix block}
$I_M^{n,v,w} \colon \Reals^n \to \Reals^n$
by
$I_M^{n,v,w} \defas  \Id_n -\frac{2 v \otimes w}{v \cdot w}$.
\end{definition}

Involutive matrix blocks are in fact involutive; for completeness, we include the following proof, adapted from \cite{levine1962construction}.

\begin{lemma}
Every involutive matrix block is involutive.
\end{lemma}

\begin{proof}
Let $n\in \Nats$ and $v, w \in \Reals^n \setminus \{0^n\}$.
The product of $I^{n, v, w}_M$ with itself is the identity:
\begin{align*}
    \bigg(\Id_n-2\frac{v \otimes w}{v\cdot w}\bigg)\bigg(\Id_n-2\frac{v \otimes w}{v \cdot w}\bigg) \hspace*{80pt}\\
    \hspace*{5pt} = \quad \Id_n - 4 \frac{v \otimes w}{v \cdot w} + 4 \frac{(v \otimes w)(v \otimes w)}{(v \cdot w)^2}
    \quad = \quad \Id_n.
\end{align*}
Hence $I^{n, v, w}_M$ is involutive.
\end{proof}

Not all involutive matrices are involutive matrix blocks; for example, the identity matrix is involutive but not an involutive matrix block.

Compared to involutive permutation blocks, involutive matrix blocks potentially allow for freer mixing between dimensions, and can also be optimized.

Note that $n$, $v$, and $w$ are uniquely determined by the function $I^{n, v, w}_M$ if we constrain $|w|=1$, and so without loss of generality we may think of it as a neural net with parameters $v$ and $w$.

Observe that involutive matrix blocks are volume preserving by the matrix determinant lemma.

\subsection{Deep involutive networks}

In order to create deep networks which have high capacity and which are exactly involutive, we want to compose several involutive blocks. For the resulting network to be involutive, we must compose them in particular ways.

\begin{definition}[Involutive network]
Let $n\in\Nats$.
    We say that a neural network is an \defn{invertible network of dimension $n$} if it
induces a bijection from $\Reals^n$ to $\Reals^n$ and its inverse is also expressible as a neural network.
    Write $\bV_n$ to denote the set of invertible neural networks of dimension $n$.

    A neural network is an \defn{involutive network of dimension $n$} if the function it induces is contained in the closure of the following operations. Write
$\bI_n$ for the set of all such functions.
    \begin{itemize}
        \item $I_F^{n,g} \in \bI_{n}$ for $g \in \bV_{n/2}$ when $n$ is even;

        \item $I_P^{n,\sigma} \in \bI_n$ for every involution $\sigma$ on $\{1, \ldots, n\}$;

        \item $I_M^{n,v,w} \in \bI_n$ for every $v, w \in \Reals^n \setminus \{0^n\}$;

        \item $\cI \circ \cJ \circ \cI \in \bI_n$ for every $\cI, \cJ \in \bI_n$;

        \item $g^{-1} \circ \cJ \circ g \in \bI_n$ for every $\cJ \in \bI_n$ and $g \in \bV_n$.
    \end{itemize}
\end{definition}

It is immediate by induction that every involutive network is involutive. Furthermore, if every block in the network is volume preserving, then the network is volume-preserving as well.

Our proof of universality considers the special case of the involutive network $g^{-1} \circ h^{-1} \circ I_P^{n,\sigma} \circ h \circ g$ for particular $n$, $g$, $h$, and $\sigma$. However, the field of deep learning has found that despite the fact that traditional neural networks with a single hidden layer are universal \citep{hornik1991approximation}, most functions of interest are learned more effectively by networks with more than one hidden layer. Therefore we recommend constructing deep involutive networks similarly, as a composition of many functions.

For a system diagram of the architecture for a typical involutive neural network,
see Fig.~\ref{fig:system-typical}.

\section{Universality of Involutive Generative Models}
\label{sec:universality}

When using a machine learning model, it is useful to know which class of functions the model can represent. In this section, we consider generative models built from deep involutive networks, and show that they are universal approximators in a certain sense. Specifically, we prove that these networks, which map a state and an auxiliary variable $(\phi,\pi) \in (\Reals^n,\Reals^m)$ to an output interpreted as another state and auxiliary variable $(\phi',\pi')$, can serve as arbitrarily good generative models of any continuous function of a Gaussian on any compact subset of $\Reals^n$. 

The theorem statement and proof are given in Section~\ref{sec:universality-proof}.

Our proof is constructive: for any desired transition $T$, we explicitly construct an involutive generative model such that transitions drawn from it are likely to be drawn, with arbitrarily high probability, from a distribution as close as desired to that of $T$. Moreover, both of these approximation parameters are explicit in the description of the generative model.
We first describe a family of involutive functions that approximate the desired cumulative distribution function, and then make use of the universal approximation theorem \cite{hornik1991approximation}, which has a constructive proof, to approximate these involutive functions by involutive neural nets.

As a consequence, involutive generative models simply match the expressive power of traditional neural generative models. There is, however, one key advantage: a standard generative model maps a state and auxiliary variable to a state: $(\phi,\pi) \mapsto \phi'$, whereas an involutive generative model produces an additional piece of information, an output auxiliary variable $\pi'$ such that the model maps $(\phi',\pi') \mapsto (\phi,\pi)$.

In other words, it produces a value for the auxiliary variable such that the model makes a backwards transition $\phi' \mapsto \phi$. This immediately gives a lower bound on the backward transition probability (via the sampling distribution for $\pi$), and is a key property we use to easily generate Metropolis--Hastings transitions satisfying detailed balance as shown in Section~\ref{sec:validity-proof}.


\subsection{Proof of Universality}
\label{sec:universality-proof}

\begin{theorem}[Involutive generative models are universal]
\label{maintheorem}
Let $n\in \Nats$, and let random variables
$\pi \sim N(0^{n+6},\Id_{n+6})$
and $\pi' \sim N(0,1)$.

For all compact sets
$\Omega \subseteq \Reals^n$,
and all continuous functions
$T \colon \Reals^n \times \Reals \to \Omega$
    there exists a sequence $\{\widehat{\cI}_m\}_{m\in \Nats}$ of involutive neural networks
that induce continuous functions $\Reals^{2n+6} \to \Reals^{2n+6}$
such that for all
$\phi \in \Omega$
the random variables
${\hcI_{m}}[\phi] \defas {\hcI_m(\phi \conc \pi})_{1..n}$ converge in distribution to
$T[\phi] \defas T(\phi \conc \pi')$, as $m\to\infty$.
\end{theorem}

We begin with an outline of the proof technique.
For each $m\in\Nats$, we aim to define a involutive neural network $\hcI_{m}$.
We begin by exhibiting an involutive function parametrized by a positive real $\fcnerr<1$ depending on $m$, such that when the auxiliary variable is treated as a random variable, the involutive function matches the cumulative distribution function of the desired state transition arbitrarily well as $\fcnerr \to 0$. Then we show that this involutive function can be uniformly approximated arbitrarily well by an involutive neural network parameterized by some other positive real $\neterr<1$ depending on both $\fcnerr$ and $m$, as $\neterr \to 0$.
(In the main proof of Theorem~\ref{maintheorem} below, we will impose tighter constraints on $\fcnerr$ and $\neterr$.)

We will use Hornik's Universal Approximation Theorem to obtain such a uniform approximation.

\begin{theorem}[{\citet[Thm.~3]{hornik1991approximation}}]
    \label{thm:Hornik-UAT}
    Let $a, b$ be positive integers,
$F \colon \Reals^a \to \Reals^b$ be a continuous function, and $\psi \colon \Reals \to \Reals$ be a continuous bounded nonconstant activation function.
    For any compact set $\Omega \subseteq \Reals^a$ and $\delta > 0$ there is a neural network consisting of a single hidden layer with activation $\psi$ that induces a continuous function
    $\widehat{F}_\delta \colon \Reals^a \to \Reals^b$ satisfying
    \[\max_{x \in \Omega} \bigl|F(x)-\widehat{F}_\delta(x)\bigr|_1 \leq \delta.\]
\end{theorem}

We now define and prove several facts about some objects that will be useful in the proof.

Define $\F \colon \Reals^{n+3} \to \Reals^{n+3}$ by
    \[
        \F(x) \defas
\begin{cases}
      0^{n+3} & \IF q \leq -\frac{1}{2}\\
      (q+\frac{1}{2})  (T(\phi,\pi)-\phi) \conc 0^3 & \IF -\frac{1}{2} < q < \frac{1}{2} \\
      (T(\phi,\pi)-\phi) \conc 0^3 & \IF \frac{1}{2} \leq q,
\end{cases}
\]
where $\phi \defas x_{1..n}$,  $\pi \defas \frac{x_{n+1}}{\fcnerr}$, and $q \defas x_{n+3} - x_{n+2}$.

Define $\G \defas \Reals^{n+3} \to \Reals^{n+3}$ by
\[
\G(x) \defas
\begin{cases}
    0^{n+3} & \IF q \leq -\frac{1}{2} \\
      (q+\frac{1}{2}) \phi' \conc 0^3 & \IF -\frac{1}{2} < q < \frac{1}{2} \\
      \phi' \conc 0^3& \IF \frac{1}{2} \leq q,
\end{cases}
\]
where $\phi' \defas x_{1..n}$ and $q \defas x_{n+3} - x_{n+2}$.

Define $\g \colon \Reals^{2n+6} \to \Reals^{2n+6}$ by
    \[
        \g(x)
        \defas
        x \odot (1^n \conc \fcnerr 1^{n+6}) +(0^{n+2} \conc 1 \conc 0^{n+2} \conc 1),
    \]
where $\odot$ denotes pointwise multiplication. Note that its inverse satisfies
    \[
        \textstyle
        \g^{-1}(x)=(x-(0^{n+2} \conc 1 \conc 0^{n+2} \conc 1)) \odot (1^n \conc \frac{1}{\fcnerr} 1^{n+6}).
    \]

Define $\h \colon \Reals^{2n+6} \to \Reals^{2n+6}$ by
    \[
        \h \defas \, \bigl(x_{1..n+3} +\G(v)\bigr) \conc v
\]
where $v \defas x_{n+4..2n+6} + \F(x_{1..n+3})$.

Note that its inverse satisfies
\[
    \h^{-1}(x) = w
    \conc \bigl(x_{n+4..2n+6} - \F(w)\bigr)
\]
where $w \defas x_{1..n+3} - \G(x_{n+4..2n+6})$.

Let $\sigma$ be the permutation of $\{1, \ldots, 2n+6\}$ that transposes $n+2$ with $n+3$ and transposes $2n+5$ with $2n+6$ (and leaves all other elements fixed).

Now consider the involutive function $\INV \colon \Reals^{2n+6} \to \Reals^{2n+6}$ defined by $\INV \defas \g^{-1} \circ \h^{-1} \circ I_P^{2n+6,\sigma} \circ \h \circ \g$.

Let $\A$ denote the event that $\PI_3 - \PI_2 > -\frac{1}{2 \fcnerr}$ and $\PI_{n+6} - \PI_{n+5} > -\frac{1}{2 \fcnerr}$ and $|\PI|<\frac{1}{\fcnerr}$ all hold.

Note that
\[
    \label{eq:neterr-assumption}
    \tag{$\dagger$}
    \lim_{\fcnerr \to 0} P(\A) = 1.
\]

\begin{lemma}
    \label{lem:conditioned}
    Conditioned on the event $\A$, we have
    \[
        \INV(\phi \conc \PI)_{1..n} = T(\phi,\PI_1)+\fcnerr \PI_{4..n+3}.
    \]
\end{lemma}
\begin{proof}
The event $\A$ fully determines the branches of $\F$ and $\G$ that are taken in the evaluation of $\INV(\phi \conc \PI)$, which enables us to simplify the expression $\INV(\phi \conc \PI)_{1..n}$ to the stated form.
\end{proof}

For $\phi \in \Omega$, define the random variable $\INVr[\phi] \defas \INV(\phi \conc \PI)_{1..n}$.
The following lemma is immediate.
\begin{lemma}
    \label{lem:rv-convergence-in-dist}
Conditioned on the event $\A$, the random variable $\INV[\phi]$
    converges in distribution to $T[\phi]$ as $\fcnerr \to 0$.
    \qed
\end{lemma}

We will show that conditioned on the event $\A$, and for appropriately small $\fcnerr$ and $\neterr$, we can approximate $\INV$ arbitrarily well with an involutive neural network $\aINVNET$.

Since $\Omega$ is a compact subset of $\Reals^n$, it is bounded, and hence contained in a ball of finite radius $r \in \Reals$. Let $\Omega^+ \subseteq \Reals^{n+3}$ be the closure of the ball of radius $7r+10$, and let $\Omega^{++} \subseteq \Reals^{n+3}$ be the closure of the ball of radius $14r+21$. The sets $\Omega^+$ and $\Omega^{++}$ are closed and bounded, hence compact subsets of $\Reals^{n+3}$.

Since $\F$ and $\G$ are continuous, by Theorem~\ref{thm:Hornik-UAT},
for any $\neterr > 0$ there are neural networks each with a single hidden layer and sigmoid activation that induce continuous functions
$\aF \colon \Reals^{n+3} \to \Reals^{n+3}$ and $\aG \colon \Reals^{n+3} \to \Reals^{n+3}$,
 satisfying
\[\max_{x \in \Omega^{++} } \bigl|\F(x)-\aF(x)\bigr|_1 \leq \neterr\]
and
\[\max_{x \in \Omega^{++} } \bigl|\G(x)-\aG(x)\bigr|_1 \leq \neterr.\]
That is, $\aF$ and $\aG$ converge uniformly to $\F$ and $\G$, respectively, on $\Omega^{++}$ as $\neterr \to 0$.


Define $\ah \colon \Reals^{2n+6} \to \Reals^{2n+6}$ by
    \[
        \ah \defas \, \bigl(x_{1..n+3} +\aG(v')\bigr) \conc v'
\]
where $v' \defas x_{n+4..2n+6} + \aF(x_{1..n+3})$.

Note that its inverse satisfies
\[
    \ah^{-1}(x) = w'
    \conc \bigl(x_{n+4..2n+6} - \aF(w')\bigr)
\]
where $w' \defas x_{1..n+3} - \aG(x_{n+4..2n+6})$.

Now form the involutive neural network that induces a function $\aINVNET \colon \Reals^{2n+6} \to \Reals^{2n+6}$ defined as follows,
\[
    \aINVNET \defas \g^{-1} \circ \ah^{-1} \circ I_P^{2n+6,\sigma} \circ \ah \circ \g,
\]
by composing the layers or neural nets corresponding to each term in the function definition.

\begin{lemma}
    \label{lem:hhat-convergence}
    As $\neterr \to 0$,
the function
    $\ah$ converges uniformly to $\h$ on $\Omega \times B_3(0^{n+6})$ and
its inverse $\ah^{-1}$ converges to $\h^{-1}$ on $\Omega^+ \times \Omega^{+}$.
\end{lemma}

\begin{proof}
First observe that $|\G(x)| \leq |x|$ and $|\F(x)| \leq 2r$, so that $|\h(x)| \leq 3|x|+4r$.

For $x \in \Omega \times B_3(0^{n+6})$, we have $|x_{n+4..2n+6}+\aF(x_{1..n+3})| < |\aF(x_{1..n+3})| + 3 < |\F(x_{1..n+3})| + 4 < 2r+4$, and hence we have $x_{n+4..2n+6}+\aF(x_{1..n+3}) \in \Omega^{++}$. Thus for $x \in \Omega \times B_3(0^{n+6})$, all applications of $\aF$ and $\aG$ in $\ah(x)$ are on points in $\Omega^{++}$ (where $\aF$ and $\aG$ converge uniformly to $\F$ and $\G$), and so $\ah$ converges to $\h$ on $\Omega \times B_3(0^{n+6})$ as $\neterr \to 0$. Furthermore, the convergence is uniform, since it is formed from sums and compositions of uniformly converging functions.

For $x \in \Omega^+ \times \Omega^{+}$ we have $|x_{1..n+3}-\aG(x_{n+4..2+6})| < |\aG(x_{n+4..2+6})| + 7r+10 < |\G(x_{n+4..2+6})| + 7r+11 < 14r+21$, and hence we have $x_{n+4..2n+6}-\aG(x_{1..n+3}) \in \Omega^{++}$. Thus for $x \in \Omega^+ \times \Omega^{+}$, all applications of $\aF$ and $\aG$ in $\ah^{-1}(x)$ are on points in $\Omega^{++}$ (where $\aF$ and $\aG$ converge to $\F$ and $\G$) and so $\ah^{-1}$ converges to $\h^{-1}$ on $\Omega^+ \times \Omega^{+}$ as $\neterr \to 0$.
\end{proof}

    Note that by Lemma~\ref{lem:hhat-convergence}, we have
\[
    \label{eq:first-fcnerr-assumption}
    \tag{$\ddagger$}
    \max_{x \in \Omega \times B_{1/\fcnerr}(0^{n+6})}\Bigabs{\ah(g(x))-\h(g(x))}<1
\]
    for sufficiently small $\neterr$.

\begin{lemma}
    \label{lem:pointwise-convergence}
Consider the random variable $x \defas \phi \conc \PI$.
Conditioned on the event $\A$, the function $\aINVNET(x)$ converges pointwise to $\INV(x)$, as $\neterr \to 0$.
\end{lemma}
\begin{proof}
Condition on $\A$ and assume $\neterr$ is sufficiently small that
\eqref{eq:first-fcnerr-assumption} holds.
    Then notice that $x \in \Omega \times B_{1/\fcnerr}(0^{n+6})$, so that $g(x) \in \Omega \times B_3(0^{n+6})$, and hence $\bigl|I_P^{2n+6,\sigma}(\ah(g(x)))\bigr|=\bigl|\ah(g(x))\bigr|<|\h(g(x))|+1 < 3|g(x)|+4r+1 < 7r+10$. Thus $I_P^{2n+6,\sigma}\bigl(\ah(g(x))\bigr) \in \Omega^+ \times \Omega^+$. Hence when evaluating $\aINVNET(x)$, the inputs to both $\ah$ and $\ah^{-1}$ are in the domains where their respective convergence properties stated in Lemma~\ref{lem:hhat-convergence} hold.
Therefore each function occurring in the definition of $\aINVNET$ convergence pointwise to the corresponding function in the definition of $\INV$. Further, all such functions are continuous, and so the result holds.
\end{proof}

For $\phi \in \Omega$, define the random variable $\aINVNETr[\phi] \defas \aINVNET(\phi \conc \pi)_{1..n}$.
Let $\xi_m \sim N(0^n,\fcnerr \Id_n)$ be an independently chosen Gaussian.

The following result is immediate from Lemma~\ref{lem:pointwise-convergence}.
\begin{lemma}
    \label{lem:rv-pointwise-convergence}
Conditioned on the event $\A$, the random variable $\aINVNETr[\phi]$
    converges in distribution to $\INV[\phi]$ as $\neterr \to 0$.
    \qed
\end{lemma}


We now prove Theorem~\ref{maintheorem}.

\begin{proof}[Proof of Theorem~\ref{maintheorem}]
    Fix $m\in\Nats$; we will define $\hcI_{m}$ such that the sequence $\bigset{\hcI_m}_{m\in \Nats}$ has the desired convergence property using Lemmas~\ref{lem:rv-convergence-in-dist} and~\ref{lem:rv-pointwise-convergence}.

We may decompose the CDF of $\aINVNETr[\phi]$ in terms of $\A$ and $\barA$:
\begin{align*}
    F_{\aINVNETr[\phi]}(\phi') =\,& P(\A)  \cdot F_{\aINVNETr[\phi] | \A}(\phi') \: \\
&+ (1-P(\A)) \cdot F_{\aINVNETr[\phi] | \barA}(\phi').
\end{align*}

Now we form a sequence $\bigset{{{\hcI_{m}}}[\phi]}_{m\in\Nats}$ and demonstrate that $\lim_{m \to \infty} F_{{{\hcI_{m}}}[\phi]}(\phi') = F_{T[\phi]}(\phi')$ at all points of continuity $\phi'$ of $T[\phi]$.

By Lemma~\ref{lem:rv-convergence-in-dist}
choose $\fcnerr$ such that
    \[
        |F_{T[\phi]}(\phi') - (F_{\INV[\phi]|\A})(\phi')|<\tfrac{1}{3m}
    \]
    holds for all points of continuity $\phi'$ of $T[\phi]$ and
    $P(\A) < \frac{1}{3m}$ holds, which is possible by \eqref{eq:neterr-assumption}.

By Lemma~\ref{lem:rv-pointwise-convergence}
choose $\neterr$ so that
\[
    \bigl|(F_{\INV[\phi]|\A})(\phi') - F_{\aINVNETr[\phi] | \A}(\phi')\bigr|  <  \tfrac{1}{3m}
\]
    holds for all points of continuity $\phi'$ of $T[\phi]$
and \eqref{eq:first-fcnerr-assumption} holds.
    This is possible because conditioned on $\A$, the random variable $\aINVNETr[\phi]$ converges in distribution to $\INV[\phi]$ and every point of continuity of $T[\phi]$ is also a point of continuity of $\INV[\phi]$.

Now define $\hcI_{m} \defas \aINVNETr$. Observe that for any point of continuity $\phi'$ of $T[\phi]$, we have
\begin{align*}
    \bigabs{F_{{{\hcI_{m}}}[\phi]}&(\phi') - F_{T[\phi]}(\phi')} \\
    &= \bigabs{F_{\aINVNETr[\phi]}(\phi') - F_{T[\phi]}(\phi')}\\
    &< \bigabs{F_{\aINVNETr[\phi] | \A}(\phi') - F_{T[\phi]}(\phi')} + \tfrac{1}{3m}\\
    &< \bigabs{F_{\INV[\phi]|\A}(\phi') - F_{T[\phi]}(\phi')} + \tfrac{2}{3m}\\
    &< \tfrac{1}{m}.
\end{align*}
Finally, considering this fact for all $m\in\Nats$, we see that ${{\hcI_{m}}}[\phi]$ converges in distribution to $T[\phi]$.
\end{proof}

Observe that universality holds even for the special case of volume-preserving involutive networks: the architecture $\aINVNET$ used in the constructive proof, defined by
    \[
        \aINVNET \defas \g^{-1} \circ \ah^{-1} \circ I_P^{2n+6,\sigma} \circ \ah \circ \g,
    \]
         has a Jacobian whose determinant has magnitude 1, since $\ah$ has the structure of an additive coupling layer and the constant Jacobians of $\g^{-1}$ and $\g$ cancel.

Having established the universality of deep volume-preserving involutive generative models, we now review the following known result showing that volume-preserving involutive functions can be used as valid proposals within an MCMC algorithm.


\subsection{Volume-preserving involutive functions lead to Metropolis--Hastings proposals satisfying detailed balance}
\label{sec:validity-proof}

The proof of detailed balance for Hybrid Monte Carlo relies on the fact that Hamiltonian dynamics composed with negating momentum is involutive and volume-preserving. It is also known that volume-preserving involutive functions lead to Metropolis--Hastings proposals satisfying detailed balance (see, e.g., \citep{marco-proof}), but we were unable to find a published argument. For completeness we describe here a procedure for using any volume-preserving involutive function as a proposal, and prove its correctness.

Our goal is to use a volume-preserving involutive function $f_\theta$ to construct a Markov process with $P_S(\phi)$ as a stationary distribution. To do this, we will find a transition such that the transition probabilities $P_M(\phi \mapsto \phi')$ satisfy the detailed balance condition:
\begin{align*}
P_S(\phi) P_M(\phi \mapsto \phi') = P_S(\phi') P_M(\phi' \mapsto \phi)
\end{align*}	

We follow the original derivation of Hybrid Monte Carlo \citep{duane1987hybrid}, since the structure of the proof is similar.

In order to make a transition, we do the following.
\begin{enumerate}
\item Introduce an auxiliary random variable $\pi$ with probability density $P_G$.
\item Propose a transition drawn from $P_H((\phi,\pi) \mapsto (\phi',\pi'))$ according to the volume-preserving involutive function $f_\theta$:
%
\begin{align*}
P_H((\phi,\pi) \mapsto (\phi',\pi')) = \delta[(\phi',\pi') - f_\theta((\phi,\pi))].
\end{align*}
\item Accept or reject that transition according to the Metropolis--Hastings acceptance criterion
\begin{align*}
P_A&((\phi,\pi) \mapsto (\phi',\pi')) \\
&= \min\Bigl(1,\frac{P_S(\phi')P_G(\pi')P_H((\phi',\pi') \mapsto (\phi,\pi))}{P_S(\phi)P_G(\pi)P_H((\phi,\pi) \mapsto (\phi',\pi'))}\Bigr).
\end{align*}
\item Marginalize over the auxiliary variable $\pi$.
\end{enumerate}

Formally, we define our transition probability by
\begin{align*}
P_M&(\phi \mapsto \phi') \\
&= \int \Bigl( 
    P_G(\pi)\, P_H((\phi,\pi) \mapsto (\phi',\pi')) \\
& \qquad \qquad P_A((\phi,\pi) \mapsto (\phi',\pi')) \Bigr)
    d\pi d\pi'.
\end{align*}

We now show that this transition satisfies the detailed balance condition.


\begin{lemma}
\label{lem:involutive-unchanged}
Applying $f_\theta$ within a Dirac $\delta$ distribution leaves the $\delta$ unchanged:
\begin{align*}
\delta[x-y] \,=\, \delta[f_\theta(x)-f_\theta(y)]
\end{align*}
\end{lemma}

\begin{proof}
For arbitrary $F$ we have
    \[
        F(y) = \int_{\Omega} \: \delta[x-y]F(x) \, dx
    \]
and
\begin{align*}
    \int_{\Omega}\: \delta[f_\theta(x)-&f_\theta(y)]F(x) \,dx \\
=\, &\int_{f_\theta(\Omega)} \: \delta[u-f_\theta(y)]F(f_\theta^{-1}(u)) \, \frac{du}{|\det f_\theta'(x)|}\\
=\, &\int_{f_\theta(\Omega)}\: \delta[u-f_\theta(y)]F(f_\theta^{-1}(u)) \,du \\ & \hspace{2cm} \text{(since $f_\theta$ preserves volume)} \\
=\, &F(f_\theta^{-1}(f_\theta(y))) \\
=\, &F(y), 
\end{align*}
and so
\[
        \int_{\Omega} \: \delta[x-y]F(x)\,dx \,=\, \int_{\Omega} \: \delta[f(x)-f(y)]F(x)\, dx. 
\]
Hence $\delta[x-y]	= \delta[f(x)-f(y)]$.
\end{proof}

\begin{lemma}
\label{lem:PH-symmetric}
$P_H$ is symmetric:
\begin{align*}
P_H((\phi,\pi) \mapsto (\phi',\pi')) = P_H((\phi',\pi') \mapsto (\phi,\pi)).
\end{align*}
\end{lemma}

\begin{proof}
We have
\begin{align*}
    P_H((\phi,\pi&) \mapsto (\phi',\pi')) \\
&= \delta[(\phi',\pi') - f_\theta((\phi,\pi))] \\
&= \delta[f_\theta((\phi',\pi')) - f_\theta \circ f_\theta((\phi,\pi))] \\ & \hspace{2cm} (\text{by
Lemma~\ref{lem:involutive-unchanged}}) \\
&= \delta[f_\theta((\phi',\pi')) - (\phi,\pi)] \\ & \hspace{2cm} \text{(since $f_\theta$ is involutive)} \\
&= \delta[(\phi,\pi) - f_\theta((\phi',\pi'))] \\ & \hspace{2cm} (\text{since} \: \delta[x] = \delta[-x]) \\
&=P_H((\phi',\pi') \mapsto (\phi,\pi)),
\end{align*}
establishing the lemma.
\end{proof}

\begin{lemma}
$P_A$ satisfies the following simpler form:
\begin{align*}
P_A((\phi&,\pi) \mapsto (\phi',\pi')) = \min\Bigl(1,\frac{P_S(\phi')P_G(\pi')}{P_S(\phi)P_G(\pi)}\Bigr).
\end{align*}
\end{lemma}

\begin{proof}
Apply Lemma~\ref{lem:PH-symmetric} to the definition of $P_A$.
\end{proof}

\begin{lemma}
\label{lem:PSPGPA-symmetric}
$P_SP_GP_A$ is symmetric:
\begin{align*}
P_S(\phi)P_G(\pi)P_A&((\phi,\pi) \mapsto (\phi',\pi')) \\
= \: &P_S(\phi')P_G(\pi')P_A((\phi',\pi') \mapsto (\phi,\pi)).
\end{align*}
\end{lemma}

\begin{proof}
We have
\begin{align*}
P_S(\phi)P_G(\pi)&P_A((\phi,\pi) \mapsto (\phi',\pi')) \\
&= P_S(\phi)P_G(\pi) \: \min\Bigl(1,\frac{P_S(\phi')P_G(\pi')}{P_S(\phi)P_G(\pi)}\Bigr) \\
&=  \: \min\bigl(P_S(\phi)P_G(\pi),P_S(\phi')P_G(\pi')\bigr) \\
&= P_S(\phi')P_G(\pi') \: \min\Bigl(\frac{P_S(\phi)P_G(\pi)}{P_S(\phi')P_G(\pi')},1\Bigr) \\
&= P_S(\phi')P_G(\pi')P_A((\phi',\pi') \mapsto (\phi,\pi)), 
\end{align*}
as desired.
\end{proof}

\begin{theorem}
The Markov chain defined by transitions $P_M$ has $P_S$ as a stationary distribution.
\end{theorem}

\begin{proof}
It suffices to show that $P_M$ satisfies the detailed balance condition
\begin{align}
P_S(\phi) P_M(\phi \mapsto \phi') = P_S(\phi') P_M(\phi' \mapsto \phi).
\label{eq:det-bal}
\tag{$\star$}
\end{align}	
We have
\begin{align*}
P_S&(\phi) P_M(\phi \mapsto \phi') \\
    &=\int \Bigl( 
    P_S(\phi)P_G(\pi) P_H((\phi,\pi) \mapsto (\phi',\pi')) \\
    & \qquad \qquad P_A((\phi,\pi) \mapsto (\phi',\pi')) \Bigr) 
    d\pi d\pi' \\
    &=\int \Bigl( 
    P_S(\phi')P_G(\pi') P_H((\phi,\pi) \mapsto (\phi',\pi')) \\
    & \qquad \qquad P_A((\phi',\pi') \mapsto (\phi,\pi)) \Bigr)
    d\pi d\pi' \\
& \hspace{3cm} \text{(by Lemma~\ref{lem:PSPGPA-symmetric})} \\
    &=\int \Bigl( 
    P_S(\phi')P_G(\pi') P_H((\phi',\pi') \mapsto (\phi,\pi)) \\
    & \qquad \qquad P_A((\phi',\pi') \mapsto (\phi,\pi)) \Bigr)
    d\pi d\pi' \\
& \hspace{3cm} \text{(by Lemma~\ref{lem:PH-symmetric})} \\
&=P_S(\phi') P_M(\phi' \mapsto \phi),
\end{align*}
and so \eqref{eq:det-bal} holds.
\end{proof}

We have established that for Markov processes generated by our transition, the desired distribution is stationary, by showing that it satisfies detailed balance.
This implies, when the chain is ergodic, that the MCMC procedure eventually generates samples which are arbitrarily close in total variation distance to the posterior distribution.

\section{Training and sampling algorithm}
\label{sec:training}

Having established the generality of our involutive MCMC procedure, we now describe a method for training optimized involutive transition kernels.

As discussed in \citep{song2017nice}, a useful transition kernel satisfies three criteria: 1) low bias in the limit; 2) fast convergence; and 3) low autocorrelation.

Volume-preserving involutive functions lead to transition kernels with zero bias in the limit (as we saw in Section~\ref{sec:validity-proof}, assuming ergodicity), so criterion 1 is satisfied.

Previous work has shown that using a ``Markov-GAN" or MGAN objective \citep{song2017nice} can satisfy criterion 2, by finding a good tradeoff between proposals being near the posterior and a high proposal acceptance rate. This objective is
\begin{align*}
\stackrel[\theta]{}{\mathrm{min \:}}
\stackrel[D]{}{\mathrm{max}}\Big\{
\bE_{x \sim p_d}[D(x)]
    -\lambda \bE_{\bar{x} \sim \chi_\theta^b}[D(\bar{x})]\hspace*{-100pt}\\
    &-(1-\lambda) \bE_{x_d \sim p_d ,\, \bar{x} \sim T^m_\theta(\bar{x} | x_d)}[D(\bar{x})]\Big\},
\end{align*}
where $\theta$ and $D$ are the parameters of the generator and discriminator, $p_d$ is the true posterior, $\chi_\theta^b$ is the distribution of samples from the Markov chain after $b$ transitions from $X$ (the given sampling distribution), $T^m_\theta(\bar{x} | x_d)$ is the distribution of samples from the Markov chain after $m$ transitions starting from a sample from the true posterior, and $\lambda$ is a free parameter.

We train using a new, computationally efficient lower-variance estimator of $\bE_{\bar{x} \sim \chi_\theta^b}[D(\bar{x})]$ which enables accurate training through multiple MCMC transitions, even when far from the posterior. This enables us to train using $\lambda=1$, which is desirable since $\lambda<1$ requires sampling from the posterior, which is not always tractable. In contrast, the training method of A-NICE-MC ignores inverse transitions, an approximation which is justified only if chains converge quickly to the posterior. We could optimize the true objective by estimating this expectation with individual samples from $\chi_\theta^b$, but we instead form a lower variance approximation. We do this by fixing a sampled set of auxiliary variables each training step, and then computing the exact expectation conditioned on the use of those auxiliary variables. This approximation may bias the true MGAN objective; it has been helpful in practice, but more analysis is required to determine when this approximation is justified. Specifically, we optimize
\begin{align*}
\stackrel[\theta]{}{\mathrm{min \:}}
\stackrel[D]{}{\mathrm{max}}\Big\{
\bE_{x \sim p_d}[D(x)]
    -\bE_{\bar{x} \sim \chi_\theta^b}[D(\bar{x})]\Big\}
\end{align*}

\begin{algorithm}[t!]
    \caption{Involutive Neural MCMC Training}
    \begin{algorithmic}[1]
        \State Let $X$ and $Y$ be sampling distributions for states and auxiliary variables respectively.
        \State Let $b$ be the number of steps of MCMC we consider during training.
        \State Let \texttt{training\_steps} be the desired number of training steps.
        \State Let $w$ be the permitted magnitude of weights for WGAN training.
        \State Initialize a neural net $D$ (the discriminator) and an involutive neural net $G$ (the generator).
        \For {\texttt{step} in \texttt{training\_steps}}
        \State Sample a true state $\widehat \phi$ from the posterior $\Phi$.
        \State Sample an initial state $\phi_0 \sim X$.
        \For {$i$ in $\{0, \ldots, b-1\}$}
        \State Sample an auxiliary variable $\pi \sim Y$.
        \State $(\phi_{i+1},\pi') \leftarrow G(\phi_{i},\pi)$
        \State $A_{i \mapsto {i+1}} \leftarrow \min\big(1,\frac{ f_\Phi(\phi_{i+1}) f_Y(\pi')}{f_\Phi(\phi_{i}) f_Y(\pi)}\big)$
        \EndFor
    	\For {$t$ in $\{0, \ldots, b\}$}
    		\State $\chi_{0,t} \leftarrow (1-A_{0 \mapsto 1})^t$
    	\EndFor
     	\For {$j$ in $\{1, \ldots, b-1\}$}
     		\State $\chi_{j,j-1} \leftarrow 0$
    		\For {$t$ in $j, \ldots, b$}
    			\State $\chi_{j,t} \leftarrow \chi_{j-1,t-1}A_{j-1 \mapsto j} $
                \Statex \hspace*{70pt} $+ \chi_{j,t-1}(1-A_{j \mapsto j+1})$
    		\EndFor
    	\EndFor
    	\State $P(i) \defas \chi_{i,b}$
        \If{$\texttt{step}\, \% \,2 == 0$ }
        \State $D_\textrm{loss} \leftarrow D(\widehat \phi) - \bE_{i \sim P}[D(\phi_i)]$
        \State Update $D$ by objective $D_{\textrm{loss}}$ via RMSProp.
        \State Clamp weights of $D$ in range $[-w,w]$.
        \Else
         \State $G_\textrm{loss} \leftarrow \bE_{i \sim P}[D(\phi_i)]$
        \State Update $G$ by objective $G_\textrm{loss}$ via RMSProp.
        \EndIf
        \EndFor \\
        \Return $G$
    \end{algorithmic}
    \label{alg:training}
\end{algorithm}

where $\chi_\theta^b$ is the distribution of states after $b$ transitions of a random Markov chain with states $\phi_i$ and transition probabilities $A_{i \mapsto j}$ defined inductively over $i$ by
\begin{align*}
\pi_i &\sim Y\\
\phi_0 &\sim X\\
\phi_{i+1} \conc \pi_{i}' &= G_\theta(S_{i},\pi_{i})\\
A_{i \mapsto j} &= \delta_{j,i+1}\min\Big(1,\frac{f_X(\phi_j)f_Y(\pi_i')}{f_X(\phi_{i})f_Y(\pi_i)}\Big),
\end{align*}
where $X$ is a sampling distribution of initial states, $Y$ is a sampling distribution for auxiliary variables, $G_\theta$ is an involutive neural network parameterized by $\theta$, $\Phi$ is the true posterior, $\delta$ is the Kronecker delta, and $f_X$ and $f_Y$ are the densities of $X$ and $Y$ respectively.

Previous work \citep{song2017nice} has also shown that training using a pairwise discriminator can reduce autocorrelation. One trains the discriminator to distinguish pairs of samples from the same chain from pairs of samples from the posterior. When it is possible to generate true independent pairs of samples from the posterior, this technique can be used to reduce autocorrelation and satisfy criterion 3.

Algorithms~\ref{alg:training} and \ref{alg:sampling} provide pseudocode describing our Wasserstein-GAN-style training procedure \citep{arjovsky2017wasserstein} and our sampling procedure. In practice, both of these algorithms should be batched over chains, and all probability calculations should be done in log space.

\begin{figure*}[t]
\centering
\setlength{\unitlength}{1mm}
\begin{picture}(160,50)
\put(40,25){\includegraphics[width=20mm]{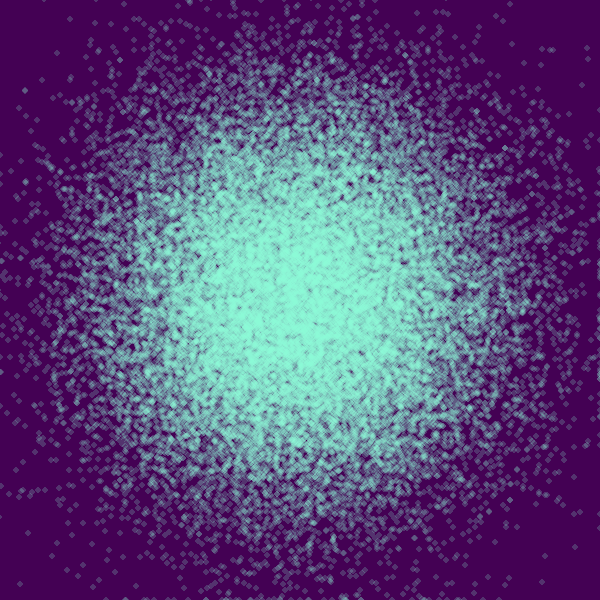}}
\put(60,25){\includegraphics[width=20mm]{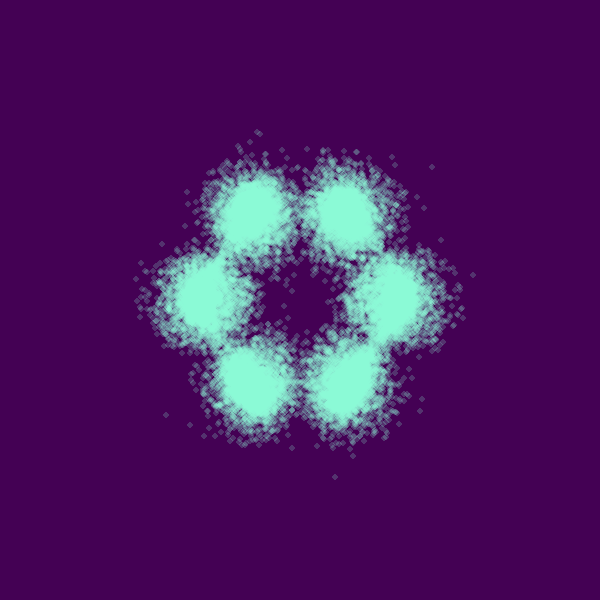}}
\put(80,25){\includegraphics[width=20mm]{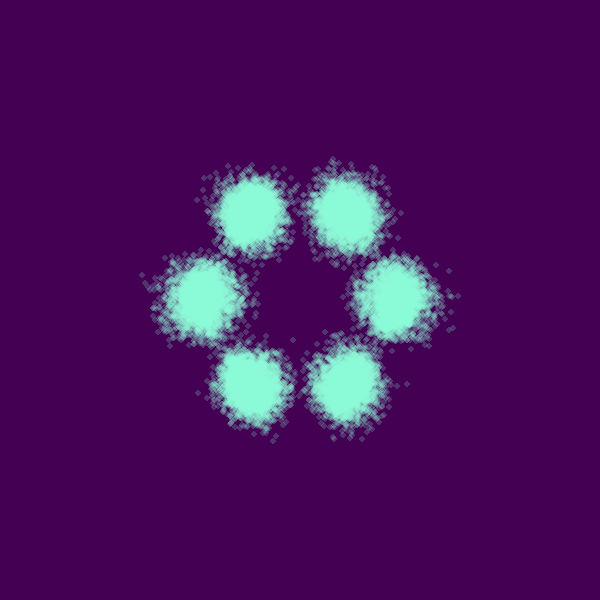}}
\put(100,25){\includegraphics[width=20mm]{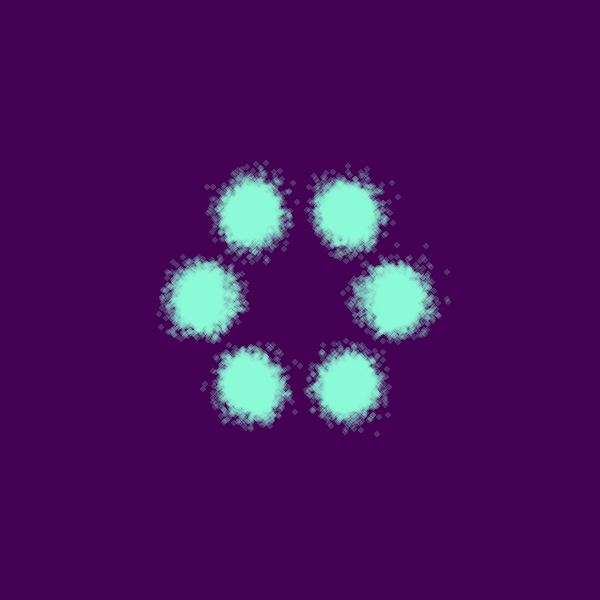}}
\put(40,5){\includegraphics[width=20mm]{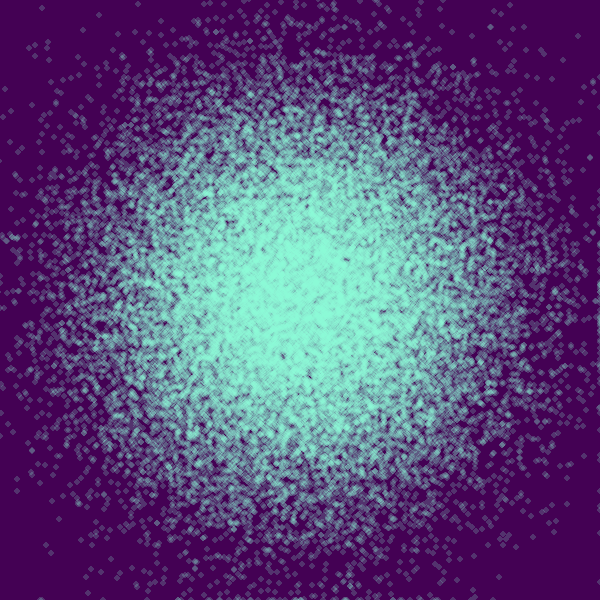}}
\put(60,5){\includegraphics[width=20mm]{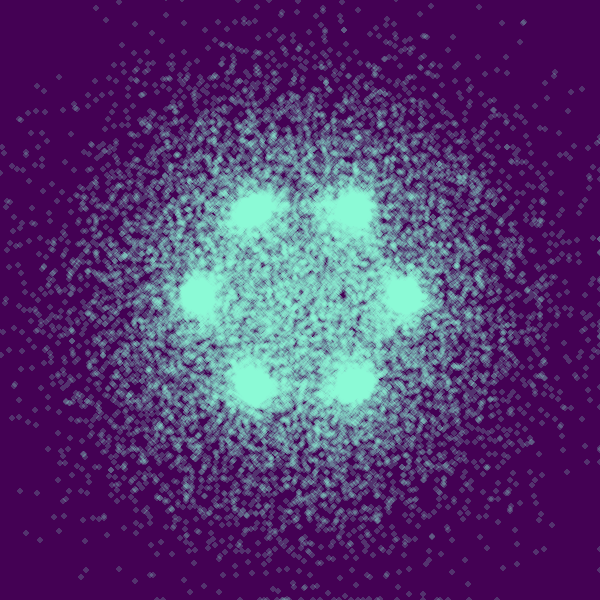}}
\put(80,5){\includegraphics[width=20mm]{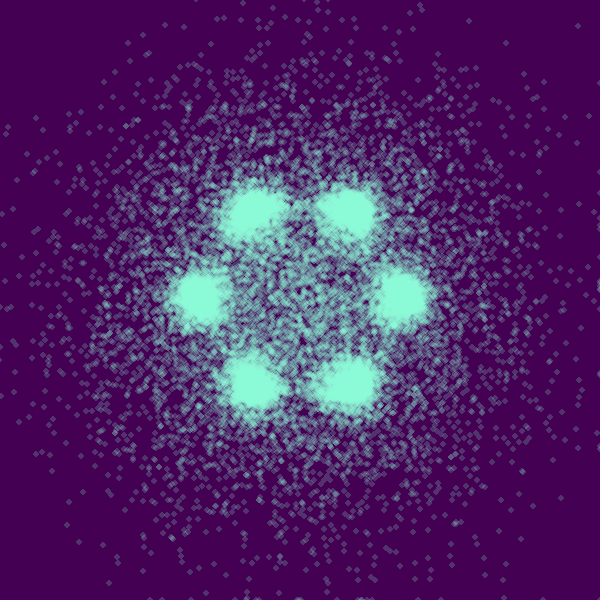}}
\put(100,5){\includegraphics[width=20mm]{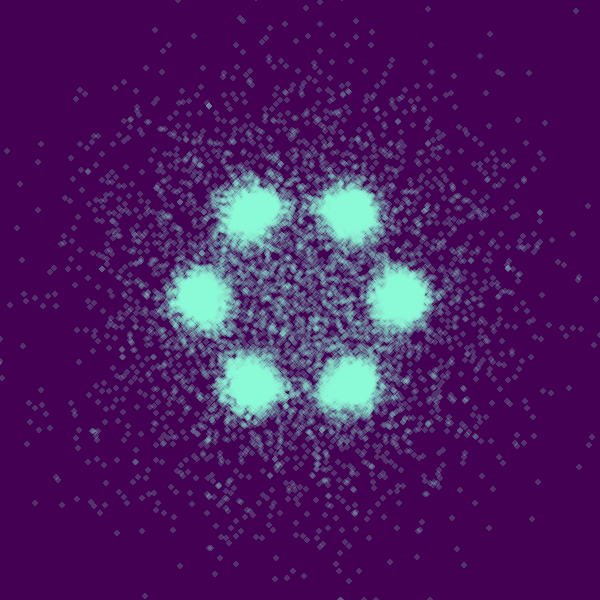}}

\put(130,15){\includegraphics[width=20mm]{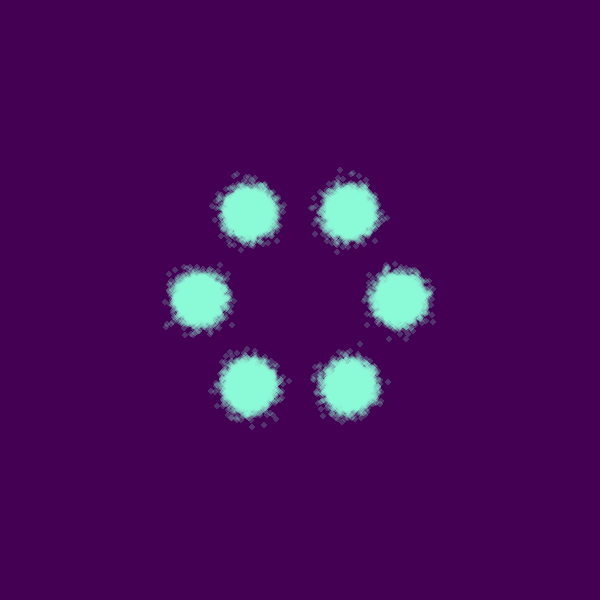}}

\put(130,10){True density}

\put(-2,34){Involutive Neural MCMC:}
\put(18,14){A-NICE-MC:}

\put(44,0){0 steps}
\put(64,0){1 step}
\put(84,0){2 steps}
\put(104,0){3 steps}
\end{picture}
\caption{Density plots of samples from A-NICE-MC and Involutive Neural MCMC. Note the outliers remaining in A-NICE-MC, which are samples for which a forward transition has never been proposed. Almost every sample from Involutive Neural MCMC is near the posterior after only one step.}
\label{fig:gaussian_densities}

\vspace{5mm}

\end{figure*}
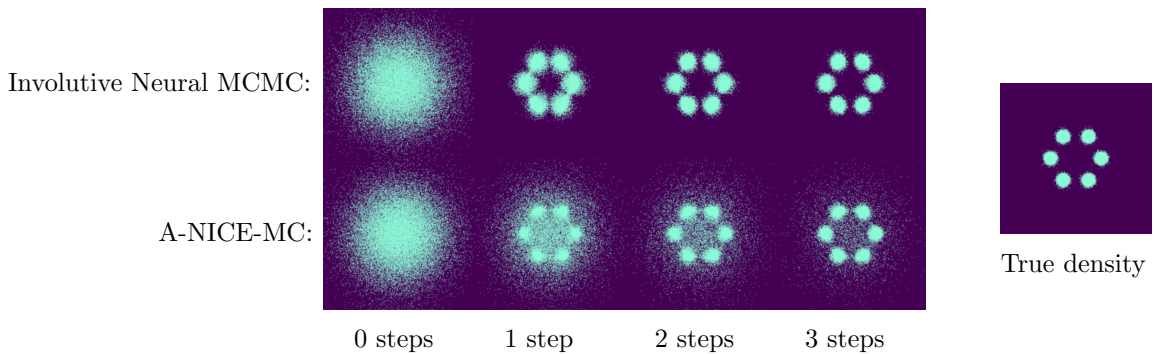

\newcommand{\propx}{119}
\newcommand{\propy}{120}
\newcommand{\xaxx}{47}
\newcommand{\xaxy}{-6}
\newcommand{\yaxx}{-6}
\newcommand{\yaxy}{48}
\newcommand{\titley}{118}
\newcommand{\corrx}{119}
\newcommand{\corry}{70}
\newcommand{\corryaxx}{-5}
\newcommand{\corryaxy}{23}
\newcommand{\corrxaxx}{50}
\newcommand{\corrxaxy}{-2}

\begin{figure*}[t]
\centering
\setlength{\unitlength}{1mm}
\begin{picture}(160,70)
\put(40,45){\includegraphics[width=20mm]{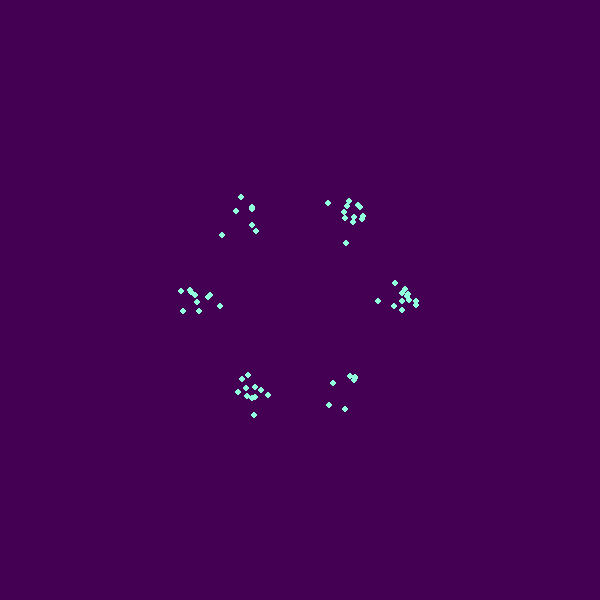}}
\put(60,45){\includegraphics[width=20mm]{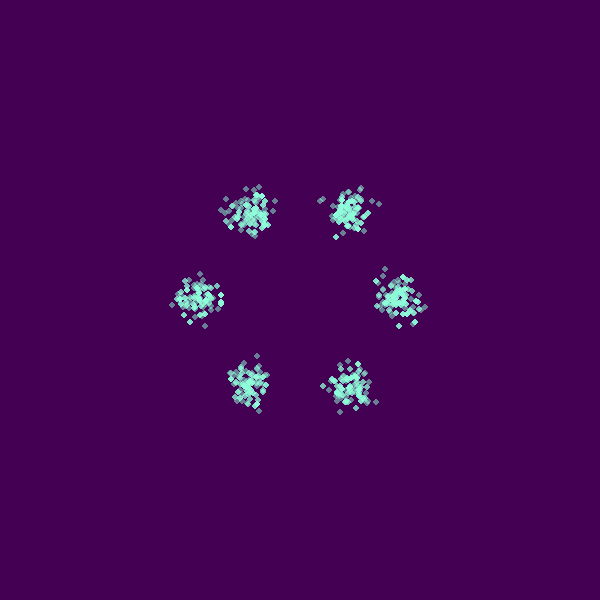}}
\put(80,45){\includegraphics[width=20mm]{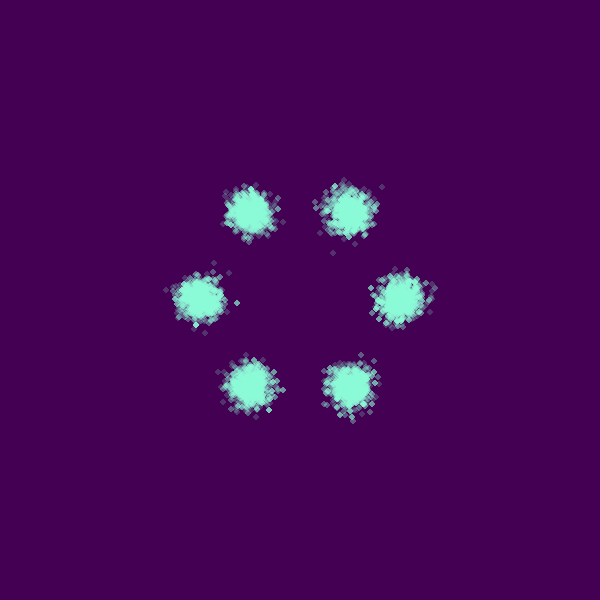}}

\put(40,5){\includegraphics[width=20mm]{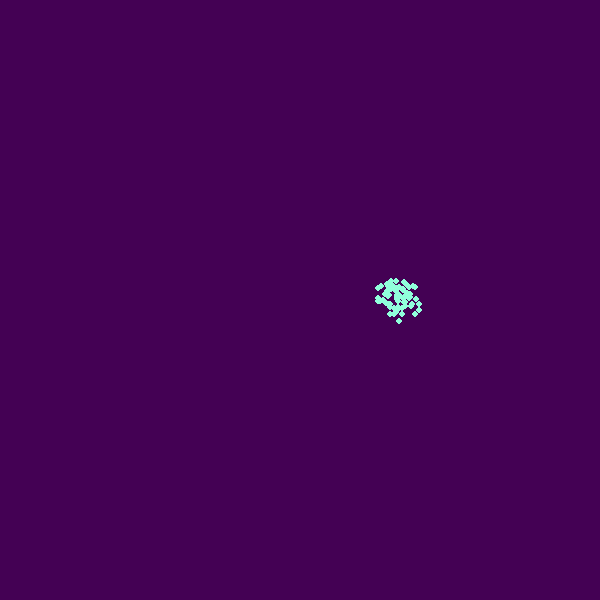}}
\put(60,5){\includegraphics[width=20mm]{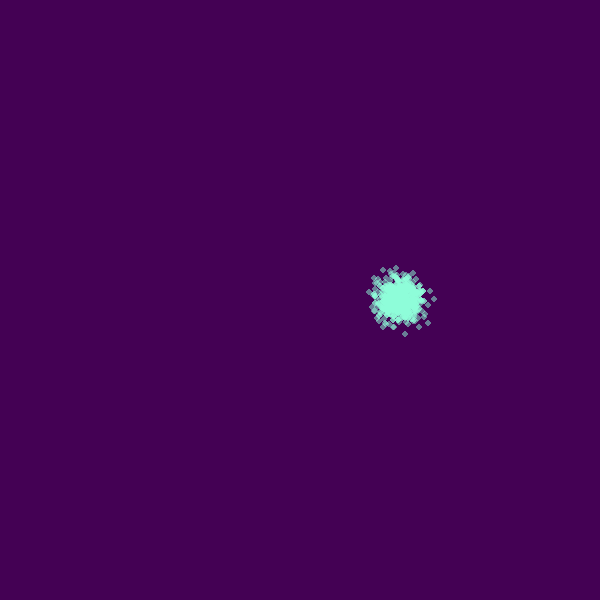}}
\put(80,5){\includegraphics[width=20mm]{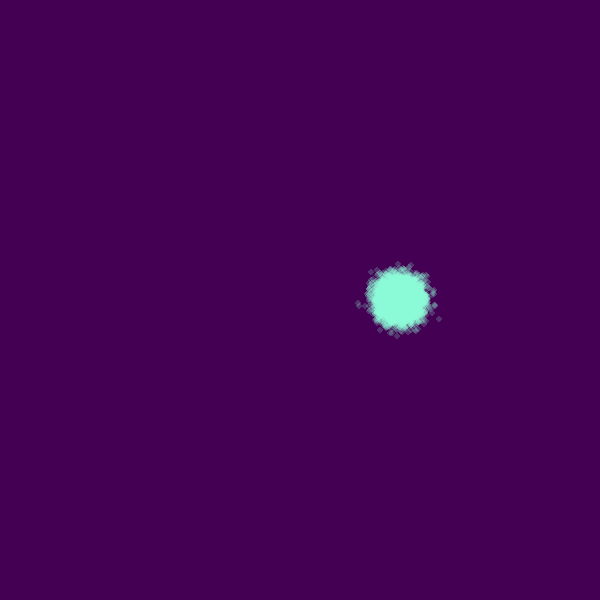}}

\put(40,25){\includegraphics[width=20mm]{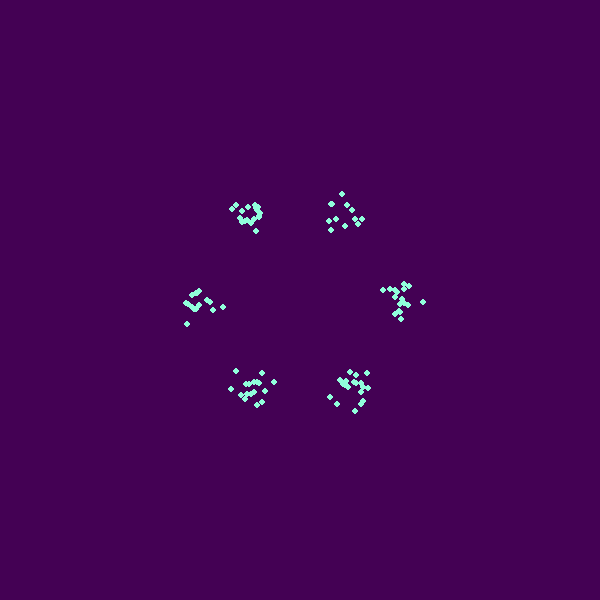}}
\put(60,25){\includegraphics[width=20mm]{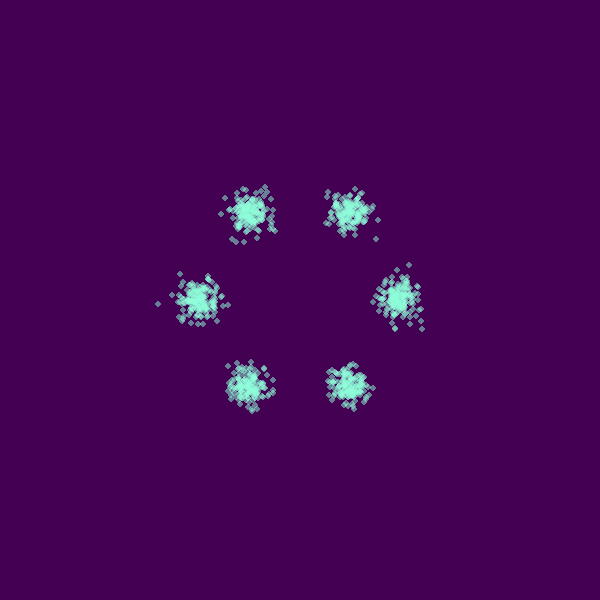}}
\put(80,25){\includegraphics[width=20mm]{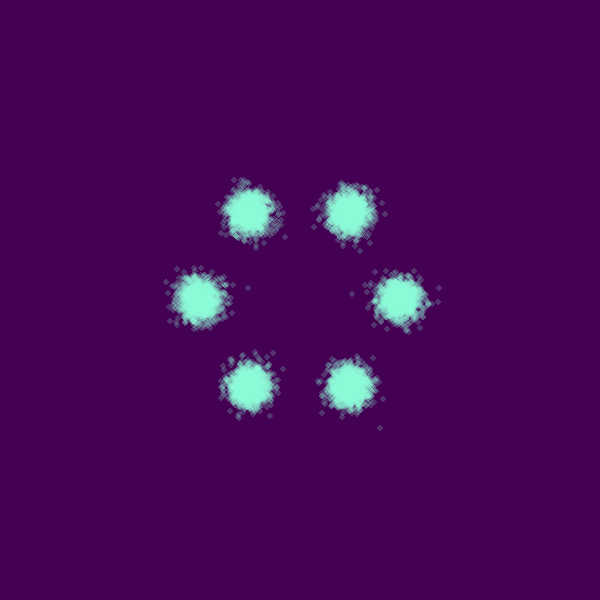}}

\put(15.5,34){True Samples:}
\put(-3,54){Involutive Neural MCMC:}
\put(5,14){Hybrid Monte Carlo:}
\put(23.8,0){Samples:}
\put(48,0){100}
\put(67,0){1000}
\put(86,0){10000}

\put(110,40){
	\begin{picture}(50,40)
	\put(2,28){Involutive Neural MCMC}
	\put(0,0){\includegraphics[width=40mm]{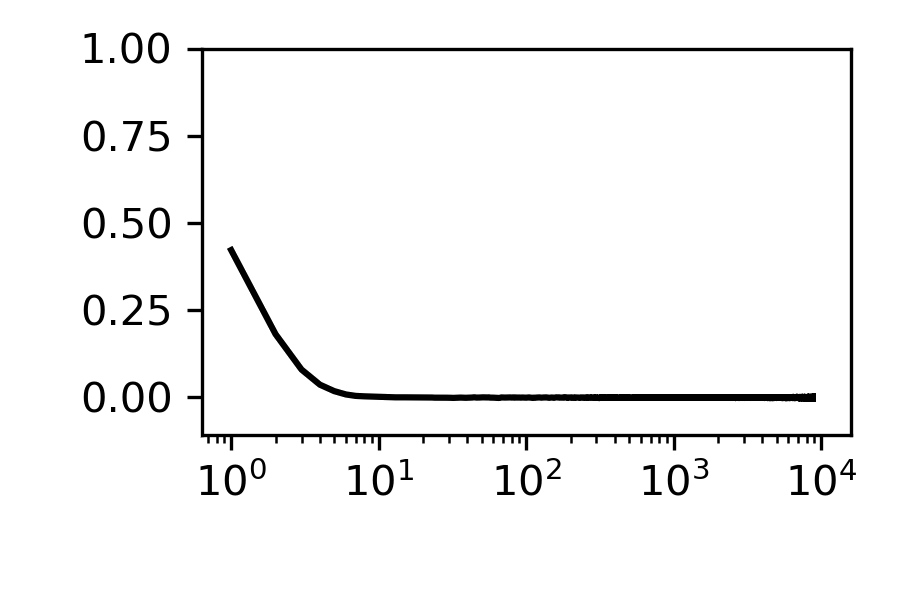}}
	\put(20,0){steps}
	\put(-2,7){$\rotatebox{90}{$x$ autocorr}$}
	\end{picture}
}

\put(110,0){
	\begin{picture}(50,40)
	\put(7.5,28){Hybrid Monte Carlo}
	\put(0,0){\includegraphics[width=40mm]{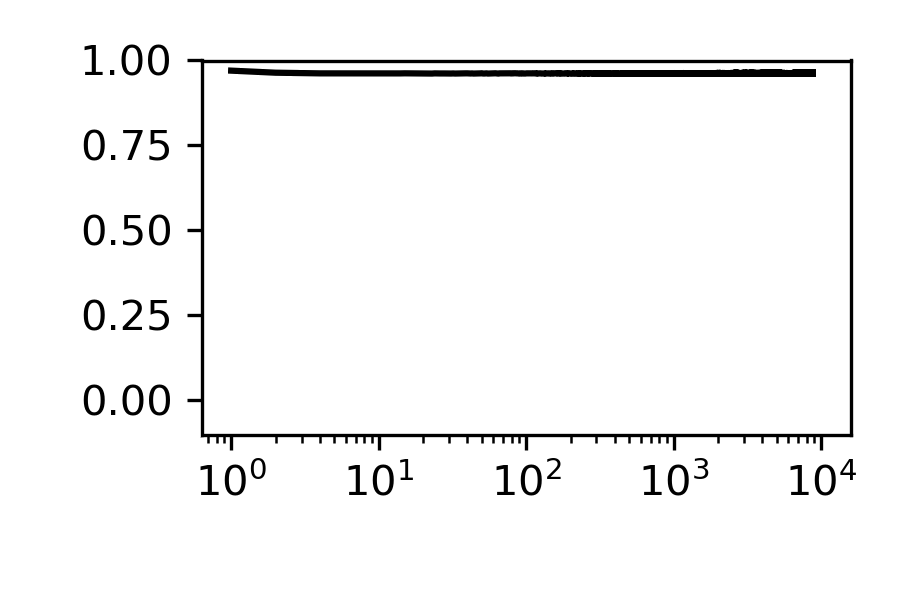}}
	\put(20,0){steps}
	\put(-2,7){$\rotatebox{90}{$x$ autocorr}$}
	\end{picture}
}

\end{picture}
\caption{Density plots of samples from a single long chain of Involutive Neural MCMC and HMC. Note that Involutive Neural MCMC mixes completely within 10 steps, whereas HMC does not mix even after 10000 steps.}
\label{fig:gaussian_chains}

\vspace{2mm}

\end{figure*}
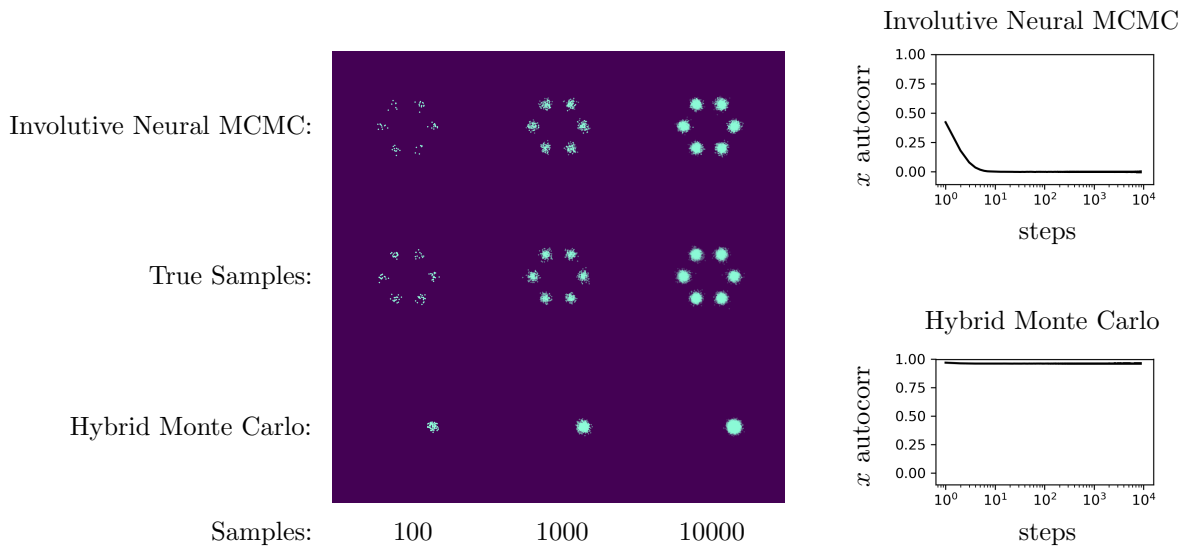

\begin{algorithm}[t!]
    \caption{Involutive Neural MCMC Sampling}
    \begin{algorithmic}[1]
    	\State Let $G$ be the generator network obtained during training.
    	\State Let $X$ and $Y$ be the sampling distributions for states and auxiliary variables respectively used during training.
    	\State Let $b$ be the number of steps of MCMC to use, which may be different than that used in training.
    	\State Sample initial state $\phi \sim X$.
        \For {$i$ in $\{1, \ldots, b\}$}
        \State Sample auxiliary variable $\pi \sim Y$.
        \State $(\phi',\pi') \leftarrow G(\phi,\pi)$
        \State $p \leftarrow \frac{ f_\Phi(\phi) f_Y(\pi')}{f_\Phi(\phi) f_Y(\pi)}$
        \State With probability $p$, update $\phi \leftarrow \phi'$.
        \EndFor \\
        \Return $\phi$

    \end{algorithmic}
    \label{alg:sampling}
\end{algorithm}

If one does not have access to a generative model for true posterior samples, one can instead bootstrap \citep{song2017nice}.

\begin{figure}[t]
\centering
\includegraphics[width=80mm]{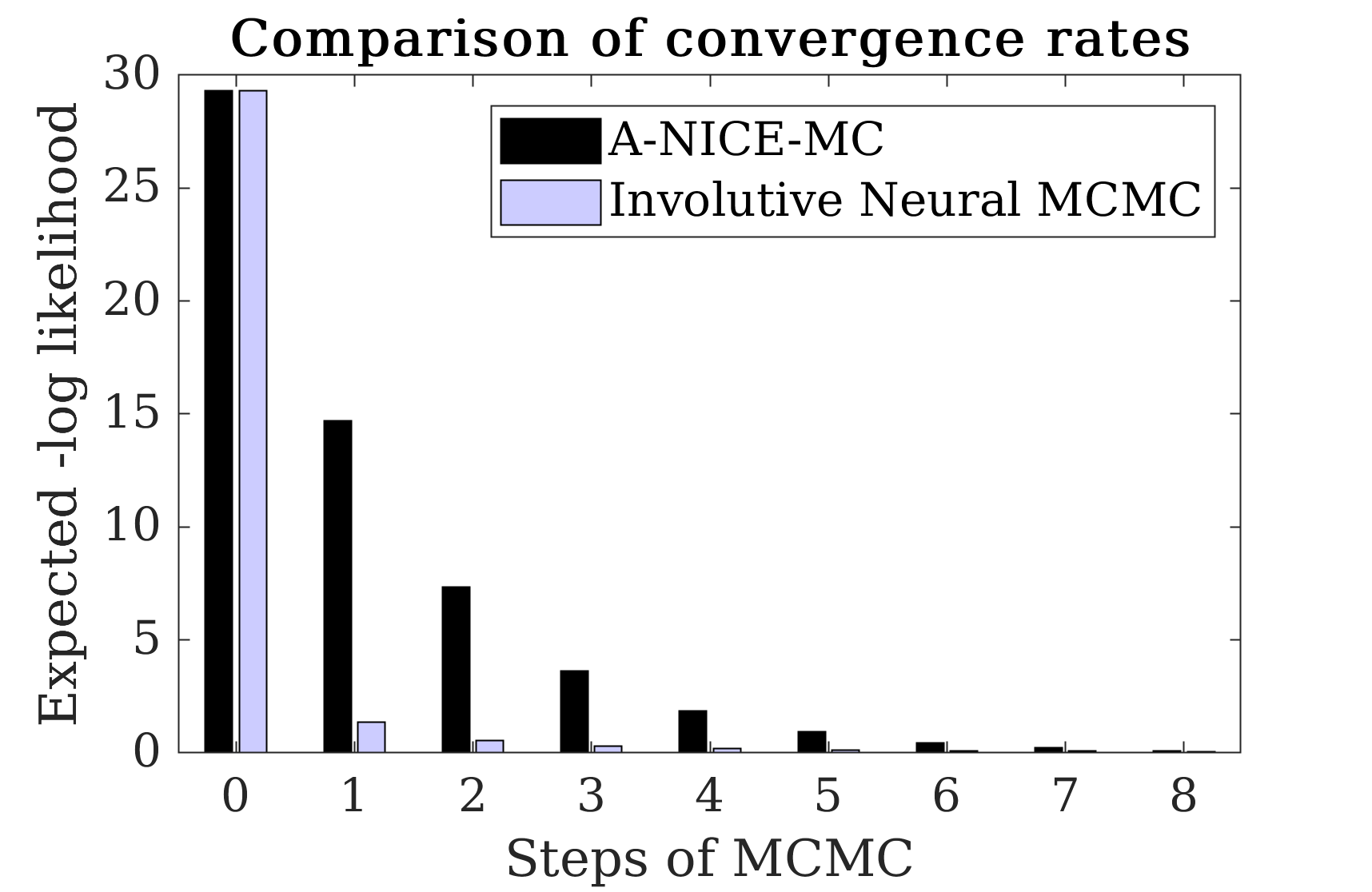}
\caption{Expected negative log likelihood of samples from A-NICE-MC and Involutive Neural MCMC relative to samples from the true posterior.}
\label{fig:gaussian-convergence}
\end{figure}

\section{Experimental Results}
\label{sec:experiments}

We train a deep volume-preserving involutive neural network to serve as a neural proposal for sampling from a mixture of six Gaussians (mog6 from \citep{song2017nice}) and compare its convergence rate to that of A-NICE-MC.

\textbf{Architecture:} In this experiment, we use states $\phi \in \Reals^2$ and auxiliary variables $\pi \sim N(0^{30},\Id_{30})$. The higher dimension of $\pi$ helps increase the width, and thus capacity, of our deep involutive network.

Our discriminator is a neural network consisting of a single hidden layer of width 64 and ReLU activation.

Our generator is an involutive neural network consisting of a symmetric composition $I_F^{n,g} \circ I_P^{n,\sigma} \circ I_F^{n,h} \circ I_P^{n,\sigma} \circ I_F^{n,g}$ where $\sigma$ is a uniform random involutive permutation (chosen once at network initialization), and $g$ and $h$ are invertible neural networks. Each of $g$ and $h$ consists of three composed invertible blocks~\citep{ardizzone2018analyzing}. The first and third invertible blocks each use as their nonlinear functions two densely connected neural networks with a single hidden layer of width 8 times its input dimension and ReLU activation. The second invertible block is a uniformly random permutation.

\textbf{Results:} Accepted transitions from both A-NICE-MC and Involutive Neural MCMC converge very quickly. However, initial proposals from A-NICE-MC have acceptance probabilities of about 50\%, whereas we observe nearly 100\% acceptance for proposals from Involutive Neural MCMC. See Figs.~\ref{fig:gaussian_densities}, \ref{fig:gaussian_chains}, and \ref{fig:gaussian-convergence} for density plots and a comparison of convergence rates.

\section{Discussion}

This paper has introduced new deep learning building blocks for
constructing involutive neural networks, deep involutive generative models, and methods to constrain these to preserve volume; it has also proved that these generative models are universal approximators for probability
kernels, and it has shown how to train and use deep volume-preserving involutive generative
models for fast neural MCMC. This paper has also demonstrated that
Involutive Neural MCMC can converge more rapidly than
A-NICE-MC, a recently introduced neural MCMC technique, and that it is
possible for Involutive Neural MCMC to switch modes more effectively
than Hybrid Monte Carlo.

Much more work is needed to empirically study the performance of
Involutive Neural MCMC on a broader class of problems and involutive
architectures. The relationship between training time, network
capacity, and convergence rate are not yet clear, even on
simple examples. We note that because deep involutive
generative models are self-inverting, it may be feasible to use
recently introduced auxiliary variable techniques to assess the
convergence rate of Involutive Neural MCMC to the true posterior, in
terms of KL divergence \citep{NIPS2017_6893}.

There is a widespread need for techniques that construct fast,
accurate MCMC proposals for broad classes of Bayesian inference
problems. Involutive Neural MCMC offers a way to learn MCMC proposals
using neural networks without requiring the ability to analytically
calculate output probability densities of those networks. A broad
class of GAN-based techniques thus become available to MCMC algorithm
designers. Also, because volume-preserving involutive generative models are universal
approximators, they can in principle learn arbitrarily good proposals
given enough network capacity and training compute
time. We hope the flexibility afforded by Involutive Neural MCMC leads to the development of many fast neural MCMC schemes for challenging inference problems.

\section*{Acknowledgements}

The authors would like to thank Ian Hunter, Sarah Bricault, Marco Cusumano-Towner, and Jayson Lynch for helpful discussions, and Sam Power for identifying an error in an earlier version of this paper. This research was supported by
Fonterra, IBM (via the MIT-IBM Watson AI Lab),
and gifts from the Siegel Family Foundation
and the Aphorism Foundation.

\bibliography{mybib}

\end{document}